\documentclass{article}
\usepackage{arxiv}
\usepackage[american]{babel}

\usepackage{natbib} 
\usepackage{mathtools} 
\usepackage{booktabs} 
\usepackage{tikz} 

\usepackage{microtype}
\usepackage{graphicx}
\usepackage{hyphenat}
\usepackage{booktabs} 
\usepackage{amsthm}
\usepackage{enumitem}
\usepackage{comment}
\usepackage{amssymb}
\usepackage{color, colortbl}
\usepackage{hyperref}
\usepackage{makecell}
\usepackage{boldline}
\setcellgapes{3pt}
\usepackage{url,wrapfig}
\usepackage{csquotes}
\usepackage{subcaption}
\usepackage{capt-of}
\usepackage{tabularx,xspace}
\usepackage{newfloat}
\usepackage{array, multirow, bigdelim, makecell, booktabs}
\usepackage{ragged2e}
\usepackage[vlined,ruled,linesnumbered,algo2e]{algorithm2e}
\hypersetup{hidelinks}
\definecolor{lppsgd}{RGB}{186, 52, 235}
\definecolor{lapsgd}{RGB}{194, 230, 53}

\DeclareFloatingEnvironment[fileext=frm,placement={!ht},name=Algorithm]{myfloat}
\usepackage[nodisplayskipstretch]{setspace}

\usepackage{amsmath,amsfonts,bm}

\newcommand{\ignore}[1]{{\iffalse #1\fi}}

\def\eqref#1{equation~\ref{#1}}
\def\Eqref#1{Equation~\ref{#1}}

\def\1{\bm{1}}

\DeclareMathAlphabet{\mathsfit}{\encodingdefault}{\sfdefault}{m}{sl}
\SetMathAlphabet{\mathsfit}{bold}{\encodingdefault}{\sfdefault}{bx}{n}

\def\gS{{\mathcal{S}}}

\def\sE{{\mathbb{E}}}

\mathchardef\mhyphen="2D

\newcounter{mynumber}

\definecolor{LightCyan}{rgb}{0.88,1,1}
\definecolor{Ao}{rgb}{0.0, 0.5, 0.0}
\definecolor{cadmiumorange}{rgb}{0.93, 0.53, 0.18}
\definecolor{cardinal}{rgb}{0.77, 0.12, 0.23}
\definecolor{byzantium}{rgb}{0.44, 0.16, 0.39}
\definecolor{gamboge}{rgb}{0.89, 0.61, 0.06}
\definecolor{goldenbrown}{rgb}{0.6, 0.4, 0.08}
\newcommand{\sone}{\textbf{S1}\xspace}
\newcommand{\stwo}{\textbf{S2}\xspace}
\newcommand{\sthr}{\textbf{S3}\xspace}
\newcommand{\sfor}{\textbf{S4}\xspace}
\newcommand{\sgd}{\textsc{Sgd}\xspace}
\newcommand{\lsgd}{\textsc{L-sgd}\xspace}
\newcommand{\alsgd}{\textsc{Lap-sgd}\xspace}
\newcommand{\plsgd}{\textsc{Pl-sgd}\xspace}
\newcommand{\palsgd}{\textsc{Lpp-sgd}\xspace}
\newcommand{\mbsgd}{\textsc{Mb-sgd}\xspace}

\newcommand{\pasmp}{\textsc{PASSM$+$}\xspace}
\newcommand{\res}{\textsc{ResNet}\xspace}
\newcommand{\wres}{\textsc{WideResNet}\xspace}
\newcommand{\bw}{\textsc{Buckwild!}\xspace}
\newcommand{\hw}{\textsc{Hogwild!}\xspace}
\newcommand{\hwp}{\textsc{Hogwild!++}\xspace}
\newcommand{\sqn}{\textsc{SqueezeNet}\xspace}
\newcommand{\rnt}{\textsc{ResNet-20}\xspace}
\newcommand{\cft}{\textsc{Cifar-10}\xspace}
\newcommand{\cfh}{\textsc{Cifar-100}\xspace}
\newcommand{\imn}{\textsc{Imagenet}\xspace}
\newcommand{\rai}{\texttt{read\&inc}\xspace}
\newcommand{\rdo}{\texttt{read}\xspace}

\newcolumntype{P}[1]{>{\RaggedLeft\hspace{0pt}}p{#1}}

\newcommand{\block}[2]{[#1,#2]\xspace}

\newcommand{\tn}{\stackrel{\sim}{\smash{{\nabla}}\rule{0pt}{1.1ex}}\kern-1ex}

\newcommand{\fadd}{\texttt{fetch\&add}\xspace}

\makeatletter
\newcommand{\nosemic}{\renewcommand{\@endalgocfline}{\relax}}
\newcommand{\dosemic}{\renewcommand{\@endalgocfline}{\algocf@endline}}%
\let\oldnl\nl
\newcommand{\nonl}{\renewcommand{\nl}{\let\nl\oldnl}}

\SetCommentSty{mycommfont}

\newtheorem{theorem}{Theorem}[section]

\newtheorem{lemma}{Lemma}[section]

\newtheorem{assumption}{Assumption}[section]

\makeatother

\title{Scaling the Wild: Decentralizing Hogwild!-style Shared-memory SGD}         
\date{}
\author{
	Bapi Chatterjee\\
	Department of Computer Science \& Engineering\\
	IIIT-Delhi, New Delhi, India\\
	\texttt{bapi@iiitd.ac.in}\\
	\And
	Vyacheslav Kungurtsev\\
	Department of Computer Science\\
	Czech Technical University in Prague\\
	\texttt{vyacheslav.kungurtsev@fel.cvut.cz } \\
	 \AND
	 Dan Alistarh\\
	 Institute of Science and Technology\\
	 Austria\\
	 \texttt{dan.alistarh@ist.ac.at} \\	
}

\begin{document}
\maketitle
	\begin{abstract}
	Powered by the simplicity of lock-free asynchrony, Hogwilld! is a go-to approach to parallelize SGD over a shared-memory setting. Despite its popularity and  concomitant extensions, such as PASSM+ wherein concurrent processes update a shared model with partitioned gradients, scaling it to decentralized workers has surprisingly been relatively unexplored. To our knowledge, there is no convergence theory of such methods, nor systematic numerical comparisons evaluating speed-up.

In this paper, we propose an algorithm incorporating decentralized distributed memory computing architecture with each node running multiprocessing parallel shared-memory SGD itself. Our scheme is based on the following algorithmic tools and features: (a) asynchronous local gradient updates on the shared-memory of workers, (b) partial backpropagation, and (c) non-blocking in-place averaging of the local models. We prove that our method guarantees ergodic convergence rates for non-convex objectives. On the practical side, we show that the proposed method exhibits improved throughput and competitive accuracy for standard image classification benchmarks on the CIFAR-10, CIFAR-100, and Imagenet datasets. Our code is available at \url{https://github.com/bapi/LPP-SGD}.
\end{abstract}

\section{Introduction}

\label{sec:intro}
Stochastic gradient descent (\sgd) is described as
 \begin{equation}\label{sgd}x_{k+1}=x_{k}-\alpha_k\tn f(x_k),
 \end{equation}
 where after $k$ iterations, $x_{k}$ is the estimate of the solution to the optimization problem: $\min_{x \in \mathbb{R}^d} f(x)$, $\alpha_k$ is step-size, and $\tn f(x_k)$ is a random variable such that $\sE[\tn f(x_k)]=\nabla f(x_k)$.  
 In machine learning applications, empirical loss  $f(x):\mathbb{R}^d\rightarrow\mathbb{R}$ of a model $x \in \mathbb{R}^d$ over a training set $\gS$ is minimized using \sgd. Commonly, we sample a (mini) batch $B_k\subseteq \mathcal{S}$ i.i.d. to compute $\tn f(x_k)=\nabla f(x_k, B_k)$. 
 
 As the size of training set $\gS$ grows, adapting \sgd to distributed -- centralized shared-memory-based and decentralized message-passing-based  -- systems becomes imperative. 
 
 In a shared-memory setting, \hw \citep{recht2011hogwild} is now a celebrated approach, wherein multiple concurrent processes (or threads) asynchronously compute $\tn f(x_k)$ to update a shared model $x \in \mathbb{R}^d$ without any lock. Subsequently, \hw saw developments such as \bw \citep{Sa2015TamingTW}, which uses lower precision arithmetic. Recently, \cite{kungurtsev2019asynchronous} presented \pasmp, which showed that on the shared-memory of a GPU, partial backpropagation with respect to (w.r.t.) blocks of a convolutional neural network (CNN) by concurrent processes can provably reduce the total training cost while offering a surprisingly competitive optimization result. However, none of these methods can scale beyond a single CPU or GPU.

In a decentralized setting, arguably, the most popular distributed SGD variant  is \textit{minibatch data-parallel SGD} (\mbsgd) \citep{bottou2012stochastic, DeanCMCDLMRSTYN12}, described as
 \begin{equation}\label{distsgd}x_{k+1}=x_{k}-\alpha_{k}\frac{1}{|\mathcal{Q}|}\sum_{q\in\mathcal{Q}} \nabla f(x_k, B^q_k).
 \end{equation}
 Essentially, $\gS$ is distributed over a set of workers/GPUs $\mathcal{Q}$, thereby $B^q_k\subseteq \mathcal{S}$ is sampled locally by a worker $q\in \mathcal{Q}$ at iteration $k$, thus $\tn f(x_k)=\frac{1}{|\mathcal{Q}|}\sum_{q\in\mathcal{Q}} \nabla f(x_k, B^q_k)$. There has been a significant interest in \mbsgd-based deep neural network training with increasingly large number of GPUs, e.g.~\cite{Goyal2017AccurateLM}, \cite{Akiba2017ExtremelyLM}, \cite{You2017ImageNetTI}, \cite{Mikami2018ImageNetResNet50TI}, etc.

Though a simple scheme to scale with the number of GPUs, \mbsgd consistently communicates gradients after each computation step, thus incurring a high communication and synchronization cost. An alternative approach is \textit{parallel or local SGD} (\lsgd)  \citep{Zinkevich2010ParallelizedSG,Zhang2016ParallelSW}, wherein workers run SGD without any communication for several local steps, after which they globally \emph{average} the resulting local models. 
Clearly, \lsgd reduces the communication cost by reducing its frequency. Other techniques such as qunatization  \citep{Alistarh2017QSGDCS} and sparsification \citep{stich2018sparsified}, applied to  message-passing distributed \sgd, primarily address the high cost of communication. Recently, these methods also have been combined with \lsgd such as \cite{0001DKD19,KoloskovaLBJS20}. Nonetheless, beyond using larger mini-batches, none of them can scale the training locally on GPUs though the contemporary GPUs efficiently host multiprocessing and thus \hw \citep{mpsnvidia,kungurtsev2019asynchronous}.   

A major motivation for decentralizing \hw-style approaches in a multi-GPU setting is that models trained by large batch regimes generally have poor quality. It  was investigated by \cite{keskar2019large}, who described this behaviour in terms of sharpness of minima correlating the minibatch-size, which degrades generalization. Our experiments also highlight this fact, see Table \ref{tab:sgd}. Leveraging multiple processes locally on each GPU can keep the batch-size small.

To our knowledge, only attempt to decentralize \hw was made by \hwp \citep{Zhang2016HogWildAN} -- run \hw locally on each CPU of a non-uniform-memory-access (NUMA) machine and synchronize by averaging on a ring-graph topology. Taking advantage of cache-locality, \hwp exhibited improved throughput over \hw for training a number of convex models. However, \hwp does not have a convergence theory and no further development to this approach is known.

In this paper, we describe the \textit{locally-partitioned-asynchronous parallel SGD} (\palsgd), a new method that employs the following techniques:
\begin{enumerate}[leftmargin=*,nosep,topsep=0pt,nolistsep,noitemsep]
\item The model on each worker is shared by multiple processes that perform lock-free partitioned gradient updates computed by partial backpropagation, thereby reduces the total gradient computation cost.
\item Dedicated processes on each worker concurrently perform the model averaging, which is in-place, non-blocking, and overlaps the gradient computations -- it significantly reduces the cost of synchronization.
\end{enumerate}
\begin{figure}[t]
	\setlength\tabcolsep{3pt} 
	\centering\small
	\begin{tabularx}{\columnwidth}{|X|X|X|X|X|X|X|X|X|X|}
		\toprule
		Method & $U$ & $B_{loc}$ & Train Loss & Train Acc. & Test  Loss & Test  Acc. & Time (Sec)& Speed-up\\
		\midrule
		\mbsgd& 1 & 128 & 0.016 &  99.75 & 0.234 & 92.95& 1754& 1.0x \\
		\mbsgd& 1 &  1024 & 0.023 &  99.51 & 0.293 & \cellcolor{orange!50}91.38& 1201& \cellcolor{green}1.46x\\
		\plsgd& 1 &   128 &0.018 & 99.69 & 0.245 & 92.98&1603& 1.09x\\
		\plsgd&  1 &  1024 &0.154 & 94.69 & 0.381 & \cellcolor{red!50}87.81&1159&\cellcolor{green}1.51x\\ \alsgd & 4& 128 & 0.012 &  99.83 & 0.270 & \cellcolor{green}92.92& 1304& \cellcolor{green}1.35x \\
		\alsgd & 6& 128 & 0.023 &  99.51 & 0.266 & \cellcolor{green}92.91& 1281& \cellcolor{green}1.37x \\
		\alsgd & 4&  256 &0.010 & 99.90 & 0.280 & \cellcolor{green}92.93&1219&  \cellcolor{green}1.44x \\
		\palsgd &4& 128 & 0.019 &  99.62 & 0.270 & \cellcolor{green}92.98& 1153& \cellcolor{green}1.52x \\
		\palsgd &6& 128 & 0.021 &  99.58 & 0.262 & \cellcolor{green}92.84& 1085& \cellcolor{green}1.62x \\
		\palsgd &4&  256 &0.019 & 99.65 & 0.267 & \cellcolor{green}92.95&1047& \cellcolor{green}1.68x \\
		\bottomrule
	\end{tabularx}\captionof{table}{\small \rnt/\cft training for 300 epochs on 2 Nvidia GeForce RTX 2080 Ti GPUs. $U$ indicates the number of concurrent gradient updater processes on each GPU. Throughout the training the local batch size is kept identical across workers. In each case, the LR is warmed up for the first 5 epochs to scale from $\alpha_{0}$ to $\alpha_{0}\times\frac{B_{loc}{\times}Q}{B_{loc{\mhyphen}base}}$, where $B_{loc{\mhyphen}base}=128$, the number of workers $Q=2$ and $\alpha_{0}=0.1$. In \plsgd, \alsgd, and \palsgd, we average the model after each gradient update for the first 150 epochs and thereafter averaging frequency is set to 16 as in \cite{Lin2020DontUL}; other hyperparameters are identical to theirs. The listed results are average of 3 runs with different seeds.}\label{tab:sgd}
\end{figure}

In the case of full (non-partitioned) gradient update by each process to the model throughout the training, we call it \textit{locally-asynchronous parallel SGD} (\alsgd). Essentially, \alsgd and \palsgd decentralize \hw and \pasmp, respectively. 

A key contribution of the presented work is the convergence theory -- under the standard assumptions, we prove that \alsgd and \palsgd  ergodically converge with standard sublinear rate for non-convex objectives. 

Experimental results in Table \ref{tab:sgd} give a glimpse of the efficacy of our algorithms. In the experiemnts, we have included \textit{post local SGD} (\plsgd), which is a modified version of \lsgd by \cite{Lin2020DontUL}.  \plsgd synchronizes the distributed models more frequently during the initial phase of training to address the issue of poor generalization of \lsgd.  

Our focus in this work is on showcasing the strength of combination of shared-memory and message-passing asynchrony for distributed SGD. Our method can be easily extended to incorporate techniques such as lower precision arithmetic, quantized and sparsified communication, etc. 

\section{System and Optimization Model}

\label{sec:model}
\begin{figure}[h]
	\setlength\tabcolsep{4pt} 
	\centering
	\begin{tabularx}{\columnwidth}{|p{3cm}|X|}
		\toprule
		Notation & Description\\
		\midrule 
		$\mathcal{Q}$, $Q$ & Set of workers (GPUs), Number of workers.\\
		$x^q$ & Local shared model on worker $q\in\mathcal{Q}$.\\		
		$u^q$ & A gradient updater process on worker $q\in\mathcal{Q}$.\\
		$U^q$, $U$ & Set of updater processes on $q$, Number thereof.\\
		$a^q$, $P^q$ & Averaging process on worker $q$, $\{\{a^q\}\cup U^{q}\}$.\\
		$x^q_i$ & $i^{th}$ block of the model (vector) $x^q$, where $i\in\mathcal{I}$.\\
		$\tn_{i}f$ & Stochastic gradient of $f$ w.r.t. $i^{th}$ block. \\
		$x^q[e]$ & $e^{th}$ index of the local model (vector) $x^q$.\\
		$x^q\block{a}{b}$ & Partition of $x^q$ from its $a^{th}$ to $b^{th}$ index.\\
		$\mathcal{C}^q$ & Shared iteration counter on worker $q\in\mathcal{Q}$.\\
		\bottomrule
	\end{tabularx}\captionof{table}{ Frequently used notations.}\label{tab:notation}
\end{figure}
We consider a $d$\hyp{}dimensional differentiable empirical risk function $f:\mathbb{R}^d\rightarrow\mathbb{R}$ and a finite training dataset $\mathcal{S}$ as described before. We aim to iteratively converge at a state $x^*$ of the model $x\in\mathbb{R}^d$ such that $x^*\in \arg\min_{x\in\mathbb{R}^d}f(x)$.

We consider a set of workers $\mathcal{Q}$ distributed over a complete graph structured network. We denote the number of workers as $Q=|\mathcal{Q}|$ and assume that $Q\ge2$. 
A worker $q\in\mathcal{Q}$ maintains a local estimate $x^q$ of the model $x$. At initialization, $x^q=x^{q'}~\forall q,q'\in \mathcal{Q}$.  We form, possibly overlapping, blocks of $x^q$ as  $x^q:=\cup_{i\in\mathcal{I}}x^q_i$, where $\mathcal{I}$ is a finite index set.

A worker $q\in\mathcal{Q}$ spawns a set of concurrent \textit{updater processes} $U^{q}=\{u^q\}$ and a designated \textit{averaging process} $a^q$. For simplicity, $|U^{q}|=|U^{q'}|=U~\forall~q,q'\in\mathcal{Q}$. On each $q\in\mathcal{Q}$, $x^q$ allows shared read/write access to the set of concurrent processes $P^{q}=\{\{a^q\}\cup U^{q}\}$. We assume that each of $u^q\in U^{q}$ on each of $q\in\mathcal{Q}$ have access to the entire training dataset $\mathcal{S}$. Thereby, a $u^q\in U^{q}$ can 
compute a minibatch stochastic gradient with respect to a block of parameters by sampling a subset $B^q\subseteq\mathcal{S}$ typically selected i.i.d. by a distribution $\mathcal{D}$ as $\tn_if(x^q):=\nabla_{x^q_i}f(x^q,B^q)$. Thus, 
\begin{equation}\label{eqn:unbiased} \begin{array}{l}\mathbb{E}_\mathcal{D}[\tn_if(x)]=\nabla_{x^q_i}f(x),\end{array} 
\end{equation}
The processes $p^q\in P^{q}$ maintain a \textit{shared iteration counter} $\mathcal{C}^{q}$, initialized to zero. The operation \rai atomically increments $\mathcal{C}^{q}$ and returns its value that was read before incrementing. The operation \rdo provides an atomic read. $\mathcal{C}^{q}$ essentially enables ordering the shared gradient updates, which determines the averaging rounds $j$. 
Finally, we have the following standard assumptions on $f$:

\text{\textbf{Lower bound.} } \begin{equation}\label{eqn:Lower bound}f(x)\ge f_m  ~\forall~x\in\mathbb{R}^d.
\end{equation}
\text{\textbf{$L$-Lipschitz Smoothness.} } \begin{equation}\label{eqn:smoothness} \|\nabla f(x)-\nabla f(y)\|\le L\|x-y\|  ~\forall~x,y\in\mathbb{R}^d. 
\end{equation}
\text{\textbf{Bounded second moment. }}
\begin{equation}\label{eqn:secondmoment} \mathbb{E}_\mathcal{D}[\|\tn_if(x)\|^2]
		\le M^2 ~\forall~x\in\mathbb{R}^d.
\end{equation}
\text{\textbf{Bounded variance. }}
\begin{equation}\label{eqn:variance} 
		\mathbb{E}_\mathcal{D}[\|\tn_if(x)-\nabla_if(x)\|^2]
		\le \sigma^2 ~\forall~x\in\mathbb{R}^d.
\end{equation}
The bounded second moment assumption (\ref{eqn:secondmoment}) is needed for \hw-style updates, see \cite{alistarh2020elastic}.

\section{Algorithm}

\label{sec:algo}
\begin{algorithm2e}[htbp]
	\nonl\textbf{Given: model $x\in\mathbb{R}^d$, training set $\mathcal{S}$, and sample processing budget $T$ minibatches of size $B_{loc}$.}\\
	\SetInd{0.5em}{0.5em}
	\nonl\textbf{\normalsize Steps of an updater process $u^q$:}\\
	Initialize $s=0$\tcp*[l]{Process-Local counter.}
	\While{$s\le T$}{
		Read and increment the shared counter $\mathcal{C}^q$: $s:=\rai(\mathcal{C}^q)$\;\label{sincread}
		Lock-free collect a  snapshot $v^{q,s}$ of $x^q$: $v^{q,s}[e]:=\rdo(x^{q}[e])$, $\forall~1\le e \le d$\;\label{line:updatesnap}
		Select a block $v^{q,s}_i$ following the \textbf{Block Selection Rule} and sample a minibatch $B^q_s\subset\mathcal{S}$\; 
		Use $B^q_s$ to compute a partial minibatch stochastic gradient w.r.t. the block $v^{q,s}_i$:  $\tn_{i}f(v^{q,s})$\;
		Lock-free update the model block $x^{q}_i$: $x^{q}_i[e]~ \mathrel{{-}{=}} ~\alpha_{s}\tn_{i}f(v^{q,s})[e]$, $\forall~1\le e \le length(v^{q,s}_i)$\;\label{line:modelupdate}
	}
\hrule
	\nonl{\normalsize \textbf{Steps of the averaging process $a^q$:}}\\
		Initialize $s_{cur}=s_{pre}=0$\tcp*[l]{Iteration counters to track updates in the model based on the shared counter $\mathcal{C}^q$ and decide whether to synchronize.}
		Initialize $j=0$\tcp*[l]{Averaging round counter.}
	\While{$s_{cur}\le T$}{
		Read the current value of the shared counter:  $s_{cur}:=\texttt{read}(\mathcal{C}^q)$\;\label{scurread}
		Compute $K_j$ corresponding to $s_{cur}$ following the \textbf{Synchronization Frequency Scheme}\;
		\If{$s_{cur}-s_{pre}\ge K_j$}{
			Update counters: $s_{pre}:= s_{cur}$; $j:=j+1$\;
			Collect a lock-free snapshot $v^{q,j}$  of $x^q$: $v^{q,j}[e]:=\rdo(x^{q}[e])$, $\forall~1\le e \le d$ \;\label{line:averagesnap}
		Synchronize across the peers in $\mathcal{Q}$ to compute all-reduced-average	$\overline{v}_j:=\frac{1}{Q}\sum_{q\in\mathcal{Q}}v^{q,j}$\;
			Compute $\Delta v^{q,j}=\overline{v}_j-v^{q,j}$\;
			In-place lock-free update: $x^{q}[e]~ \mathrel{{+}{=}} ~\Delta v^q[e]$, $\forall~1\le e \le d$\;\label{line:averageupdate}
		}
	}
	\caption{\small Steps of \palsgd on a $q\in\mathcal{Q}$.}\label{alg:algo}
\end{algorithm2e} 
We describe \palsgd in Algorithm \ref{alg:algo}, of which \alsgd is a special case. For simplicity, the steps taken by individual processes are given as sequential pseudocode, however, by implementation, each $a^q, u^q\in P^q$ work concurrently over the same address space associated with a worker $q\in\mathcal{Q}$, and all $q\in\mathcal{Q}$ run in parallel. 

\paragraph{Block partitions.} Considering a model  implemented as an ordered list of tensors, a block is essentially a sub-list of successively ordered tensors. More formally, consider a local model estimate $x^q\in \mathbb{R}^d$ and a set of positive integers $\{b_0,b_1,\ldots,b_U\}$ s.t.
$b_0=0,~b_U=d\text{ and }\sum_{i=1}^{U}(b_i-b_{i-1})=d$. We form the blocks $\{x^q_i\}_{i\in\mathcal{I}}$ ranked by the index set $\mathcal{I}=\{0,1,\ldots,U\}$ s.t. $x^q:=\cup_{i\in\mathcal{I}}x^q_i$ as the following:
\begin{align}\label{eqn:partition}
	x^q_i&:=\left\{ \begin{array}{lr}x^q\block{1}{d},&~\text{if}~i=0,\\
	x^q\block{b_{i-1}+1}{b_i},& ~\text{if}~i>0.\end{array}\right.
\end{align} 
Essentially, block $x^q_0$ represents the \textit{entire model vector} $x^q$ whereas other blocks form \textit{disjoint partitions} of $x^q$. 

For a CNN model, $x^q$ represents list of its layers (tensors). We partition it into blocks $\{x^q_1,\ldots,x^q_U\}$ of lengths $\{b_i-b_{i-1}\}_{i=1}^U$. The choice of the integers $b_i$s are such that the aggregated cost of concurrent partial backpropagation starting from the output layer to these blocks could be minimized \citep{kungurtsev2019asynchronous}.
\paragraph{Block Selection Rule.} In \alsgd, each $u^q\in U^{q}$ selects $x^q_0$, i.e. full model, throughout the training. In \palsgd as well, during the initial phase of training, say, while $s\le T_{st}$ for $s$ read at line \ref{sincread} in Algorithm \ref{alg:algo}, each $u^q$ selects the full model $x^q_0$. After that, in alternate iterations, select the full model $x^q_0$ and the block $x^q_{rank(u^q)}$: $rank(u^q)\in\{1,\ldots,U\}$. 
 
\paragraph{Discussion on the Rule.} As described in \cite{kungurtsev2019asynchronous}, we have total potential savings of $F\frac{(c-1)}{2}$ floating point operations (flops) in $c$ concurrent partial gradient computation steps by $c$ processes, where $F$ flops is the cost of full gradient computation. \cite{kungurtsev2019asynchronous} also showed that partial gradient updates throughout the training causes poor optimization results. We observed that at line 3 Algorithm \ref{alg:algo} by selecting $s\le T/10$ roughly, it consistently resulted in the desired optimization quality. Thus, for about $9T/20$ iterations we perform partial gradient computation. 
\paragraph{Synchronization Frequency Scheme.} Our approach to fix the synchronization frequency $K_j$ is similar to that of post local SGD \citep{Lin2020DontUL}: $K_j=1$ until $T/2$ minibatches are processed, i.e. until $s_{cur}$ reads up to $T/2$ at line \ref{scurread} in Algorithm \ref{alg:algo}, where $T$ is the total number of minibatches to be processed; thereafter we increase $K_j$ to $H>1$. Typically, we choose $H=8, ~16$, etc.

\section{Experiments}

\label{sec:experiments}
In this section, we present experimental evaluation of \alsgd and \palsgd against \mbsgd and \plsgd schemes. 

\textbf{Distributed training.} We use CNN models \rnt \citep{He2016DeepRL}, \sqn \citep{Iandola2017SqueezeNetAA}, and \wres-16x8 \citep{Zagoruyko2016WideRN} for the 10-/100-class image classification tasks on datasets \cft/\cfh \citep{Krizhevsky2009LearningML}. We also train \res-50 for a 1000-class classification problem on the \imn \citep{Russakovsky2015ImageNetLS} dataset. We keep the sample processing budget identical across the methods i.e. the training procedure terminates on performing an equal number of epochs (full pass) on the training set. 
 Our data sampling is i.i.d. -- we partition the sample indices among the workers that can access the entire training set; the partition indices are reshuffled every epoch following a seeded random permutation based on epoch-order.  

\textbf{Platform specification.} We use the following specific settings: (a) \textbf{S1:} two Nvidia GeForce RTX 2080 Ti GPUs, (b) \textbf{S2:} four Nvidia GeForce RTX 2080 Ti GPUs, (c) \textbf{S3:} two \textbf{S2} settings connected with a 100 GB/s infiniband link, and (d) \textbf{S4:} eight Nvidia V100 GPUs on a cluster. Other system-related details are given in the Appendix \ref{append:exp}. 
%

\textbf{Implementation framework.} We used open-source Pytorch 1.5 \cite{paszke2017automatic} library. For cross-GPU/cross-machine communication we use NCCL \citep{ncclnvidia} primitives provided by Pytorch. \mbsgd is based on \texttt{DistributedDataParallel} Pytorch module. The \plsgd implementation is derived from the author's code \cite{localsgd} and adapted to our setting. 
The \texttt{autograd} package of Pytorch allows specifying the leaf tensors of the computation graph of the CNN loss function w.r.t. which gradients are needed. We used this functionality in implementing partial gradient computation in \palsgd.

\textbf{Locally-asynchronous Implementation.} \ignore{We use process-level concurrency of Python multiprocessing package for asynchronous updates and averaging. This is out of the purview of the thread-level global interpreter lock (GIL) \cite{gil}, thus non-blocking. To specify,} 
We use a process on each GPU, working as a parent, to initialize the CNN model and share it among the spawned child-processes. The child-processes work as updaters $u^q\in U^q$ on a GPU $q \in \mathcal{Q}$ to compute the gradients and update the model. Concurrently, the parent-process, instead of remaining idle, as it happens commonly with such designs, acts as the averaging process $a^q$, thereby productively utilizing the entire address space occupied over the GPUs. The parent- and child-processes share the iteration and epoch counters.

\textbf{Hyperparameters (HPs).} Each of the methods use identical momentum and weight-decay for a given CNN/dataset case; we rely on their previously used values \citep{Lin2020DontUL}. The LR schedule for \mbsgd and \plsgd are identical to \cite{Lin2020DontUL}. For the \palsgd and \alsgd, we used cosine annealing schedule without any intermediate restart \citep{Loshchilov2017SGDRSG}. Following the well-accepted practice, we \textit{warm up} the LR for the first 5 epochs starting from the baseline value used over a single worker training. In some cases, a grid search \citep{Pontes2016DesignOE} suggested that for \palsgd warming up the LR up to 1.25$\times$ of  the \textit{warmed-up} LR of \alsgd for the given case, improves the results.  Additional details are in Appendix.

\textbf{\large Experimental Results:}\\*
Result tables use these abbreviations: Tr.L.: Training Loss, Tr.A.: Training Accuracy, Te.L.: Test Loss, Te.A.: Test Accuracy, and $T$: Time in Seconds. The asynchronous methods have inherent randomization due to process scheduling by the operation system. Irrespective of that, in each micro-benchmark we take the mean of 3 runs with different random seeds unless otherwise mentioned. 

\begin{figure}[h]
	\begin{minipage}[b]{0.5\columnwidth}
		\setlength\tabcolsep{3pt} 
		\centering\small
		\begin{tabularx}{0.5\columnwidth}{X|X|X|X|X}
			\toprule
			$U$ & $B_{loc}$ & \textbf{Tr.A.} & \textbf{Te.A.} & $T$\\
			\midrule
			4& 128 & 99.62 & 92.81& 1304\\
			6& 128 & 99.85 & 92.92& 1266 \\
			8& 128 & 99.62 & 92.64& 1259\\
			4&  256 & 99.72 & 92.83&1219\\
			6&  256 & 99.90 & 92.24&1166\\
			\bottomrule
		\end{tabularx}\captionof{table}{\small \alsgd.}\label{tab:concurrency}
	\end{minipage}	
	\hfill%
	\begin{minipage}[b]{0.5\columnwidth}
		\setlength\tabcolsep{3pt} 
		\centering\small
		\begin{tabularx}{0.5\columnwidth}{X|X|X|X|X}
			\toprule
			$K$ & $\mathbf{P}$ & \textbf{Tr.A.} & \textbf{Te.A.} & $T$\\
			\midrule
			1& 150 & 99.76 & 93.01& 1266\\
			4& 150 & 99.85 & 92.78& 1263 \\
			8& 150 & 99.62 & 93.08& 1262\\
			8&  0 & 99.78 & 92.83&1263\\
			16&  150 & 99.90 & 93.16&1261\\
			\bottomrule
		\end{tabularx}\captionof{table}{\small \alsgd.}\label{tab:avg}
	\end{minipage}	
\end{figure}

\textbf{Concurrency on GPUs.} GPU on-device memory is a constraint. However, once the data-parallel resources get saturated, allocating more processes degrades the optimization quality vis-a-vis performance. For instance, as shown in Table \ref{tab:concurrency} -- the average performance for different $U$ and $B_{loc}$ for training \rnt/\cft by \alsgd on setting \sone -- there is apparently no advantage beyond 6 and 8 concurrent processes with batch-sizes 128 and 256, respectively. 
\begin{figure}[h]
	\setlength\tabcolsep{3pt} 
	\centering\small
	\begin{tabularx}{\columnwidth}{p{12.8mm}|X|p{2.8mm}|X|X|X|X|X}
		\toprule
		\textbf{Method} & $B_{loc}$ & $U$ & \textbf{Tr.L.} & \textbf{Te.L.} & \textbf{Tr.A.} & \textbf{Te.A.} & \textbf{$T$}\\
		\midrule
		\alsgd &       64 &            4 &     0.007 &    1.341 &      {99.98} &     {75.97} &     {1564} \\ 
		\palsgd &       64 &           3 &     0.008 &    1.554 &      \textbf{99.97} &     \textbf{75.42} &     \textbf{1376} \\ 
		\mbsgd &      128 &           - &     0.007 &    1.348 &      99.97 &     73.29 &     2115 \\ 
		\mbsgd &      512 &            - &     0.034 &    1.918 &      99.64 &     63.74 &     2076\\
		\plsgd &      128 &           - &     0.020 &    1.354 &     99.81 &     72.38 &      2108\\
		\plsgd &      512 &             - &     0.393 &    1.758 &      88.91 &     60.47 &     1976\\
		\bottomrule
	\end{tabularx}\captionof{table}{\small \wres-16x8 on \cfh over the setting \textbf{S3}.}\label{tab:Arch_wnet168_Data_cifar100}
\end{figure}

\begin{figure}[h]
	\setlength\tabcolsep{3pt} 
	\centering\small
	\begin{tabularx}{\columnwidth}{p{12.8mm}|X|p{2.8mm}|X|X|X|X|X}
		\toprule	\textbf{Method} & $B_{loc}$ & $U$ & \textbf{Tr.L.} & \textbf{Te.L.} & \textbf{Tr.A.} & \textbf{Te.A.} & \textbf{$T$}\\
		\midrule
		\alsgd &      128 &              4 &     0.004 &    0.385 &      {99.95} &     {92.48} &     {1251} \\ 
		\palsgd &      256 &               3 &     0.007 &    0.355 &      \textbf{99.89} &     \textbf{92.55} &     \textbf{1027} \\ 
		\mbsgd &     1024 &       - &     0.003 &    0.325 &      99.97 &     91.50 &     1580 \\ 
		\mbsgd &      128 &               - &     0.005 &    0.294 &      99.93 &     92.41 &     2243\\
		\plsgd &     1024 &              - &     0.005 &    0.329 &      99.93 &     91.99 &     1202\\
		\plsgd &      128 &              - &     0.006 &    0.283 &      99.91 &     92.51 &     1843\\
		\bottomrule
	\end{tabularx}\captionof{table}{\small \sqn on \cft over the setting \textbf{S1}.}\label{tab:Arch_squuezenet_Data_cifar10}
\end{figure}

\begin{figure*}[h]
	\begin{minipage}{\textwidth}
		\begin{minipage}[b]{\textwidth}
			\centering
			\includegraphics[width=\linewidth]{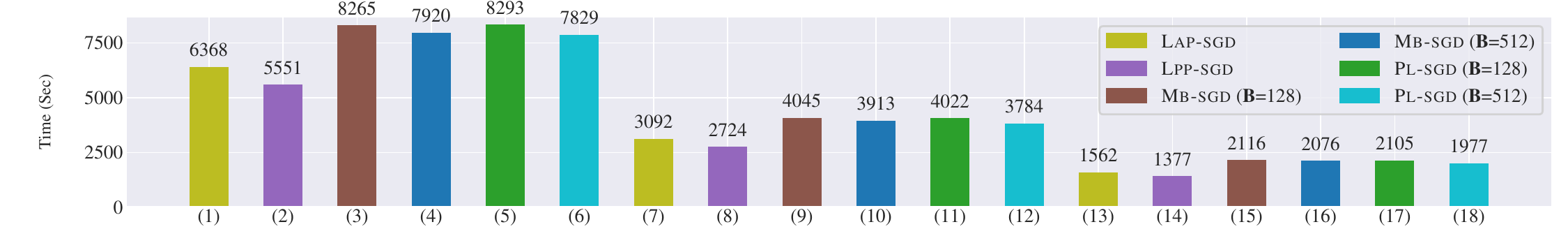}
		\end{minipage}
		\begin{minipage}[b]{\textwidth}
			\centering\small
			\setlength\tabcolsep{3pt} 
\begin{tabularx}{\textwidth}{||p{13mm}|p{6mm}|p{3mm}||p{5mm}|p{3mm}|X|X|X||p{6mm}|p{3mm}|X|X|X||p{6mm}|p{3mm}|X|X|X||}
\toprule
\textbf{Method}& $B_{loc}$ & $U$ & {} & $\mathbf{Q}$ & \textbf{Tr.A.} & \textbf{Te.A.} & \textbf{$T$} 
& {} & $\mathbf{Q}$ & \textbf{Tr.A.} & \textbf{Te.A.} & \textbf{$T$}
& {} & $\mathbf{Q}$ & \textbf{Tr.A.} & \textbf{Te.A.} & \textbf{$T$}\\
\midrule
\cellcolor{lapsgd!50}\alsgd & 64 &  4 &  (1)&      2 &     \cellcolor{lapsgd!50}100.0 &     \cellcolor{lapsgd!50}94.92 &     \cellcolor{lapsgd!50}6368  &(7)&       4 &     \cellcolor{lapsgd!50}100.0 &     \cellcolor{lapsgd!50}94.91 &     \cellcolor{lapsgd!50}3092   &(13)&   8 &     \cellcolor{lapsgd!50}100.0 &     \cellcolor{lapsgd!50}94.69 &     \cellcolor{lapsgd!50}1562 \\
\cellcolor{lppsgd!50}\palsgd & 64 &  3 &  (2)&       2 &      \cellcolor{lppsgd!50}99.97 &     \cellcolor{lppsgd!50}94.89 &     \cellcolor{lppsgd!50}5551   &(8)&        4 &      \cellcolor{lppsgd!50}99.96 &     \cellcolor{lppsgd!50}94.87 &     \cellcolor{lppsgd!50}2724  &(14)&    8 &      \cellcolor{lppsgd!50}99.99 &     \cellcolor{lppsgd!50}93.78 &     \cellcolor{lppsgd!50}1377 \\
\mbsgd &128 & - &   (3)&      2 &     100.0 &     93.98 &     8265  &(9)&       4 &     100.0 &     94.22 &     4045   &(15)&   8 &     100.0 &     93.83 &     2116 \\
\mbsgd &512 & - &   (4)&      2 &      95.68 &     88.42 &     7920    &(10)&       4 &      99.87 &     89.46 &     3913  &(16)&   8 &      87.54 &     73.97 &     2076 \\
\plsgd &128 & - &   (5)&      2 &      99.99 &     94.21 &     8293  &(11)&       4 &      99.98 &     94.02 &     4022  &(17)&   8 &      99.98 &     94.18 &     2105 \\
\plsgd &512 & - &   (6)&      2 &      95.34 &     88.33 &     7829  &(12)&       4 &      98.52 &     84.45 &     3784  &(18)&   8 &      89.19 &     81.80 &     1977 \\
\bottomrule
\end{tabularx}

			\captionof{table}{\small\wres-16x8/\cft training in the settings \sone: $\mathbf{Q}=2$, \stwo: $\mathbf{Q}=4$, and \sthr: $\mathbf{Q}=8$.}\label{tab:scale}
		\end{minipage}
	\end{minipage}
\end{figure*}
\textbf{Synchronization Frequency.} As mentioned, we follow the strategy introduced by \plsgd of setting $K=1$ for an initial set of iterations followed by increasing $K$. However, unlike \plsgd, intuitively, with our non-blocking synchronization among GPUs, the increment events on the shared-counter $\mathcal{C}^q$ are ``grouped" on the real-time scale if the updater processes $u^q$ do not face delays in scheduling, which we control by allocating an optimal number of processes to maximally utilize the compute resources. For instance, Table \ref{tab:avg} lists the results of 5 random runs for 300 epochs over setting \sone of \rnt/\cft training with $B_{loc}=128$ and $U=6$ with different synchronization frequencies $K$.  $\mathbf{P}$ denotes the number of epochs up to which $K$ is necessarily $1$. This micro-benchmark indicates that the overall latency and the final optimization result of our method may remain robust under small changes in $K$, which may not necessarily be the case with \plsgd. 

\textbf{Scalability.} Table \ref{tab:scale} presents the results of \wres-16x8/\cft training in the settings \sone, \stwo, and \sthr. We observe that in each setting the relative performance of different methods are approximately consistent. We observe a perfect scalability with the number of GPUs. Interestingly, increasing $B_{loc}$ from 128 to 512 does not improve the latency by more than $\sim$4\% of \mbsgd and \plsgd -- this non-linear scalability of data-parallelism was also observed by \cite{Lin2020DontUL}. However, there is clear degradation in optimization results by larger $B_{loc}$ for \mbsgd and \plsgd. On the other hand, \alsgd and \palsgd not only speeds-up the training but also consistently improve on the baseline generalization accuracy.

\ignore{In particular, we note the following: (a) reduced communication cost helps \plsgd marginally outperform \mbsgd, (b)  interestingly, increasing $B_{loc}$ from 128 to 512 does not improve the latency by more than $\sim$4\%; this non-linear scalability of data-parallelism was also observed by \cite{Lin2020DontUL} , (c) each of the methods scale by more than 2x as the implementation is moved to \stwo, which has 40 CPU cores, from \sone which has 12 CPU cores, furthermore, this scalability is approximately 2x w.r.t. performance on \sthr in comparison to \stwo, this shows that for each of the methods we utilize available compute resources maximally, (d) in each setting \alsgd achieves $\sim$30\% better throughput compared to \mbsgd standard batch, (e) in each setting \palsgd outperforms \alsgd by $\sim$12\% making it the fastest method, (f) the training and test accuracy of local large minibatch methods are poor, and (g) \alsgd and \palsgd consistently improve on the baseline generalization accuracy.
}
\textbf{Other \cft/\cfh Results.} Performance of the methods on \wres-168/\cfh and \sqn/\cft for 200 epochs are available in Tables \ref{tab:Arch_wnet168_Data_cifar100} and \ref{tab:Arch_squuezenet_Data_cifar10}, respectively. 
The relative latency of the methods are as seen in Table \ref{tab:scale}, whereas, in each case \palsgd recovers or improves the baseline results.

\begin{figure}[h]
	\setlength\tabcolsep{3pt} 
	\centering\small
	\begin{tabularx}{\columnwidth}{p{12.8mm}|X|p{2.8mm}|X|X|X|X|X}
		\toprule
		\textbf{Method} & $B_{loc}$ & $U$ & \textbf{Tr.L.} & \textbf{Te.L.} & \textbf{Tr.A.} & \textbf{Te.A.} & \textbf{$T$}\\
		\midrule
		\alsgd &       32 &             4 &     0.823 &      77.97 &    1.215 &     {76.58} &    54243 \\
		\palsgd &       32 &             4 &     0.834 &      77.31 &    1.257 &     \textbf{76.11} &    \textbf{48066} \\
		\mbsgd &      128 &             - &     0.809 &      78.63 &    1.220 &     76.07 &    52395 \\
		\plsgd &      128 &             - &     0.818 &      78.22 &    1.223 &     76.12 &    50402 \\		
		\bottomrule
	\end{tabularx}\captionof{table}{\small \res-50/\imn  Training results.}\label{tab:rn50}
\end{figure}

\textbf{Imagenet Training Results.} We used the setting \sfor for training \res-50 on 1000-classes Imagenet dataset.
\ignore{In comparison to \cite{Goyal2017AccurateLM,You2017ImageNetTI,Akiba2017ExtremelyLM,Mikami2018ImageNetResNet50TI}, etc., wherein batch-size up to multiple thousands are uses, our experiments definitely face platform constraints.} The computational resources of a GPU for \res-50/\imn training are almost saturated, which can be seen in higher latency of \alsgd in Table \ref{tab:rn50}. Yet, we observe  that \palsgd exhibits improved performance with competitive generalization accuracy indicating the efficacy of partial back-propagation.
 
\begin{figure}[h]
	\setlength\tabcolsep{6pt} 
	\centering\small
	\begin{tabularx}{\columnwidth}{p{14.5mm}|p{3mm}|X|X|X|X}
		\toprule
		\textbf{Method} & $U$& $B_{loc}$ & Tr.A. & Te.A. & $T$\\
		\midrule
		\mbsgd& 1 & 1024 & 95.84 & 85.67& 1324\\
		\plsgd& 1 & 1024 & 93.67 & 84.69& 1304 \\
		\alsgd & 4 &256 & 99.87 & 92.90& 1232\\
		\palsgd& 4 &256 & 99.59 & 92.83&1071\\
		\bottomrule
	\end{tabularx}\captionof{table}{\small LARS performance.}\label{tab:lars}
\end{figure}
\textbf{LR Tuning strategy.} It is pertinent to mention that the  techniques, such as LARS \citep{You2017LargeBT}, which provide an adaptive LR strategy for large-batch regimes over many GPUs, with each worker locally processing a small batch, are insufficient in the case $B_{loc}$ is increased. For example, see Table \ref{tab:lars}, which lists the average performance of 5 runs on the setting \sone to train \rnt/\cft using \mbsgd and \plsgd combined with LARS with $\eta=0.001$ \citep{You2017LargeBT}. We proportionately scale the LR from $\alpha_{0}$ to $\alpha_{0}\times\frac{B_{loc}{\times}Q}{B_{loc{\mhyphen}base}}$, where $B_{loc{\mhyphen}base}=128$, $Q=2$ and $\alpha_{0}=0.1$. Evidently, LARS did not help \mbsgd and \plsgd in checking the poor generalization due to larger $B_{loc}$.

\section{Convergence Theory}

\label{sec:convergence}
Note that, unlike \alsgd or \palsgd, the previous works on Local SGD analyses, e.g.~\cite{zhou2017convergence, stich2018local, HaddadpourKMC19}, etc. did not have to consider either lock-free multiple local gradient updates, or in-place overlapping model updates by averaging. Clearly, significant challenges arise in our case because (a) the snapshots used for gradient computations are inconsistent due to the lock-free concurrent gradient updates, and (b) they also could potentially have been read from the local shared-memory before the last averaging across workers could take place -- a novel challenge for any analysis. 

Nevertheless, we firstly need to fix a sequence of model estimates for convergence. To this end, we define an abstraction comprising two auxiliary variables as described below. For simplicity we use a constant LR $\alpha$ in this discussion. 
\paragraph{Minor iteration view:} In the local lock-free setting of each worker $q\in\mathcal{Q}$, we follow the elastic consistency model of \cite{alistarh2020elastic} for shared-memory asynchronous SGD. Now, because we are working at the level of blocks, we define the operation $-|_i$ to denote subtraction from the $i^{th}$ block. For example, if $a=\begin{pmatrix} 1 & 2 & 3 & 4 \end{pmatrix}^T$ and $b=\begin{pmatrix} -10 & -20 \end{pmatrix}^T$, we have that $a-|_{\block{2}{3}} b = \begin{pmatrix} 1 & 12 & 23 & 4 \end{pmatrix}^T$. With that, considering worker $q$, the ${t+1}^{st}$ \textit{minor iteration view} after the $j^{th}$ averaging round, assuming that the $s^{th}$ ($s$ is a read value of the shared counter $\mathcal{C}^q$) snapshot $v^{q,s}$ was used for partial stochastic gradient computation is defined by the following iteration:
\begin{equation}\label{eqn:minoriter}
		\begin{array}{l}	w_{j,t+1}^{q}:=w_{j,t}^{q} -|_i \alpha \tn_{i} f(v^{q,s}).\end{array}
	\end{equation}
\paragraph{Major iteration view:} Let $K_j^q$ be the final minor iteration count on worker $q$ before the $(j{+}1)$st synchronization round, we define the $j^{th}$ \textit{major iteration view} as 
\begin{equation}\label{eqn:majoriter}
	\begin{array}{l}	\bar{x}_{j}:=\frac{1}{Q}\sum_{q=1}^Q w_{j,K_j^q}^{q}.\end{array}
	\end{equation}
This $\bar{x}_{j}$ is the \textit{averaged iterate} that we take as the reference sequence in our convergence theory. At initialization, $\bar{x}_{0}=v^{q,0}~\forall~ q\in\mathcal{Q}$. The update rule of $\bar{x}_{j}$ can be written as
\begin{equation}\label{eq:update}
\begin{array}{l}\bar x_{j+1}=\frac{1}{Q}\sum_{q=1}^Q\left[\bar x_j -|_i
	\alpha\sum_{t=0}^{K_j^q}  \tn_{i}f(v^{q,s})\right].	\end{array}
\end{equation}
To streamline these views, we use the following notions:
\begin{enumerate}[leftmargin=*,nosep,topsep=0pt,nolistsep,noitemsep]
	\item The \textbf{Snapshot Order} of a minor iteration view is the value of $\mathcal{C}^q$ before the snapshot collection.
	\item The \textbf{Update Order} of a minor iteration view is defined as the order of the {atomic read steps} of the model before starting its update with a local gradient by a process.
	\item The \textbf{Iteration Order} of the $t^{th}$ minor iteration after the $j^{th}$ major iteration is the tuple $(j,t)$.
\end{enumerate}

More detail on these orders are provided in Appendix \ref{append:proof}. It is pertinent to mention here that the selection of a block $i\in\mathcal{I}$ depends on the worker $q$ and the iteration order $(j,t)$. Thus, precisely expressing, $i\equiv i(q,j,t)$. Table \ref{tab:extnotation} lists out some extra notations that we will be frequently using in this discussion in addition to those listed out in Table \ref{tab:notation}. 
\begin{figure}[h]
	\setlength\tabcolsep{2pt} 
	\centering
	\begin{tabularx}{\columnwidth}{|p{2.5cm}|X|}
		\toprule
		Notation & Description\\
		\midrule 
		$w^{q}_{j,t}$, $w_{u}^{q}$, $w^{q,s}$& The minor iteration view with iteration order $(j,t)$, with update order $u$, with snapshot order $s$. \\ %
		$v^{q}_{j,t}$, $v_{j,t}^{q}[e]$ & Snapshot vector to compute a partial stochastic gradient for the minor iteration $(j,t)$, its $e^{th}$ index. \\
		$\tn_{i}f(v^{q}_{j,t})$& Stochastic gradient of $f(v^{q}_{j,t})$ w.r.t. its $i(q,j,t)^{th}$ block. \\
		\bottomrule
	\end{tabularx}\captionof{table}{Additional frequently used notations.}\label{tab:extnotation}
\end{figure}%

With these views, in order to obtain a converging sequence of averaged iterates, we need to place a reasonable bound on the inconsistencies due to gradient computations for minor iterations using the snapshots that may have components read before \textit{the latest major iteration}. To this end, we describe a stochastic delay model as described below.

\ignore{The locally asynchronous gradient computation using inconsistent snapshots of the shared-memory model on each worker introduces delays in terms of number of iterations before the updates. }
Because the read and updates at the level of indexes of the model vector are atomic, we have the following:
\begin{equation}\label{eq:minorview}
	\begin{array}{l}
		v^{q}_{j,t}[e]\in \cup_{k\in\{0,1,\ldots,u\}} \{w^{q}_{k}[e]\},
	\end{array}
\end{equation}
where $e\in i(q,j,t)$ and  $i(q,j,t)\subseteq\{1,2,\ldots,d\}$ and $u$ is the update order of minor iteration $(j,t)$. 
Essentially, relation (\ref{eq:minorview}) captures the fact that a component of an inconsistent snapshot would belong to a \textit{historical minor iteration view}. Notice that, as the lock-free read and updates of the shared model work at the level of indexes, we have an advantage of dropping the dependence of the expressions over the blocks $i\in\mathcal{I}$ in the convergence analysis. 
We can now define the event describing relation (\ref{eq:minorview}) as the following. 

\paragraph{Delay Event and Probability.} Define $\mathbb{I}^{q,k,e}_{j,t}$ as the event that $v^{q}_{j,t}[e]=w^{q}_{k}[e]$, and $p^{q,k,e}_{j,t}$ as its probability, for some $j\ge 0$, $0\le t \le K_j^q$, and  $0\le k \le u$, where $u$ is the update order of minor iteration view $(j,t)$. We call $\mathbb{I}^{q,k,e}_{j,t}$ a \textit{delay event}.
Notice that, defining this event captures the possibility that $v^{q}_{j,t}$ could have been collected  at a minor iteration before the synchronization round $j$. 

\paragraph{Assumptions Governing the Asynchronous Computation.} In the non-convex setting,  in general, we can simply hope that the objective function decrease achieved by local asynchronous SGD update with a gradient computed after an averaging step outweighs the potential worst-case increase that might take place due to an update by a gradient computed before the averaging step. In order to derive an expected decrease in the objective, we need to bound the probability of an increase, which means bounding the probability that a component of the snapshot was collected at a minor iteration older than the previous averaging.

The significant quantity that separates these expected favorable and unfavorable iterations is,
\[
\mathbb{P}\left[\cap_e \cup_{l\ge k}^u \mathbb{I}^{q,l,e}_{j,t}\right],
\]
where $u$ is the update order of minor iteration $(j,t)$ and $k$ is the update order of the minor iteration just after the $j^{th}$ synchronization round. This quantity is \emph{the probability that every component of the inconsistent snapshots were due to the minor iterations after  $a^q$ performed averaging}.

\begin{assumption}\label{as:delays}
	For all $(j,t,q)$ it holds that the quantity
	\[
	\bar{\mathbf{p}}^{(j,t,q)}:=\mathbb{P}\left[\cap_e \cup_{l\ge k}^u \mathbb{I}^{q,l,e}_{j,t}\right]
	\]
	satisfies that there exists $p_0>0$ such that $\bar{\mathbf{p}}^{(j,t,q)} > p_0$ for all $(j,t,q)$.
\end{assumption}


Assumption \ref{as:delays} implies that an inconsistent snapshot read at line \ref{line:updatesnap} to compute a partial stochastic gradient does not consist of components updated by minor iterations too differing in their update orders.  

\paragraph{Elastic Consistency.} For lock-free local updates, we can make a statement to bound the error in the sense of \textit{Elastic Consistency} of \cite{alistarh2020elastic}. In particular, if $s$ is the snapshot order of minor iteration $(j,t)$, 
it holds that,
\begin{equation}\label{eq:consistency1}
	\mathbb{E}\left[\left\|w^{q}_{j,t}-v^{q,s}\right\| | \cap_e \cup_{l\ge k}^u \mathbb{I}^{q,l,e}_{j,t}\right]  \le \alpha^2 B^2.
\end{equation}
In inequation (\ref{eq:consistency1}) the event $\cap_e \cup_{l\ge k}^u \mathbb{I}^{q,l,e}_{j,t}$ used for conditioning is precisely the cases where the snapshots for computing partial stochastic gradients were taken after the last averaging round, i.e., if every component satisfies this condition then the elastic consistency bound holds.

Note that under this condition, the maximum delay in terms of number of local asynchronous SGD updates is $\tau_{max}=\bar{K}:=\max_{j,q} K^q_j$. Thus, following  Appendix B.6 of~\cite{alistarh2020elastic}, the following Lemma defines the Elastic consistency constant:
\begin{lemma}
	$B=\sqrt{d}\bar{K}M$, where $d$ is the dimension of the model, $\bar{K}$ is the maximum minor iteration count defined above, and $M$ is as defined for the second moment bound in Equation \ref{eqn:secondmoment}.
\end{lemma}

We now present the main theorem proposing the ergodic convergence for a general non-convex objective:

\begin{theorem}\label{th:convergence}
	Under the Assumptions ~\ref{eqn:unbiased},~\ref{eqn:Lower bound},~\ref{eqn:smoothness},~\ref{eqn:secondmoment}, ~\ref{eqn:variance}, defining $f_m$, $L$, $M$, $\sigma$ and the stochastic delay model described by Assumption \ref{as:delays}, it holds that,
\begin{equation}\label{eq:conv}
\begin{array}{l}
\frac{1}{Q} \sum_{j=1}^J \sum_{q=1}^Q \sum_{t=0}^{K_j^q}\left\{
[\alpha C_1-\alpha^2 C_2]\mathbb{E}\left\|\nabla_{i} f(v^{q,s})\right\|^2\right.\\
\left.+[\alpha C_3-\alpha^2C_4] \mathbb{E}\left\|\nabla_{i} f(\bar{x}_j)\right\|^2 \right\}\le f(\bar{x}_0)-f_m
\end{array}
\end{equation}
where $C_1$, $C_2$, $C_3$, and $C_4$ depend on $L$, $B$, $M$, $\sigma$ and probabilistic quantities defining the asynchronous computation. 

Thus for any set of such constants, if sufficiently small $\alpha=\Theta(1/\sqrt{J})$ is chosen 
then Algorithm~\ref{alg:algo} ergodically converges with the standard  $O(1/\sqrt{J})$ rate for nonconvex objectives.
\end{theorem}
The core idea to prove Theorem \ref{th:convergence} is arguing that the worst-case increase in the objective due to "bad gradient updates", which are computed over the snapshots collected before the last averaging, are bounded. With such a bound we then derive the convergence of the "good updates" along the lines of standard local SGD analysis such as \cite{zhou2017convergence}. We defer the detailed proof to Appendix \ref{append:proof}.

\textbf{\large Distributed Speedup and Local scalability:}

\textbf{Distributed speedup} with the number of workers $Q$ can be ascertained from the dependence on $Q$ in Theorem \ref{th:convergence}. In implementation, the proposed methods exhibit an almost linear data-parallel speedup w.r.t. $Q$, as shown in Table \ref{tab:scale}, where the effect of synchronization and communication costs is limited. 

Formally, now consider, for instance~\cite{yu2019linear}; in their presentation of interpreting \textit{linear speedup} from the convergence theory, the learning rate $\alpha$ is declared to be proportional to $\sqrt{Q}$. In effect, since, in practice, the gradient updates are averaged across workers, the effective $\alpha$ is always divided by $Q$. So then we can consider, in the context of Theorem~\ref{th:convergence}, scaling $\alpha$ by $Q$ to attempt obtaining linear speedup effectively. However, this change would add an additional factor of $Q$ to the $\alpha^2$ terms in (\ref{eq:conv}), which will negatively balance the effect away from linear speedup, the extent of which will depend on the particular constants.   

Additional results indicating linear speedup for local SGD appear in~\cite{HaddadpourKMC19}. However, they use an additional assumption of a $PL$ inequality, which can relate the gradient norm to the objective value even for nonconvex functions. While their technical analysis is impressive and interesting, we restrict our theoretical results to considering the most general case.

\textbf{Local Scalability}, i.e., speedup with the number of local updater processes $U$, of asynchronous shared memory SGD can, in theory, achieve linear speedup while $U$ is bounded by $O(\sqrt{T})$, where $T$ is the total number of minibatch updates~\citep{LianHLL15}, if the delay probabilities are not affected by the number of processes. However, in practice, scheduling an additional process on a GPU has substantial overhead, which does influence the delay probabilities, thus limiting the speedup. For \palsgd the more partitions the blocks are separated across, the smaller the size of gradient update, which decays the theoretical speedup. However, in practice, the reduced gradient computation cost of partial backpropagation handsomely improves the balance towards faster performance. 

\section{Related Work}

\label{sec:related}	
In the shared-memory setting, \hw \cite{recht2011hogwild} is now a celebrated approach to scale SGD. However, it remains applicable to a centralized setting of a single worker, thus, is not known to have been practically utilized for fast large scale CNN training. Its success led to designs of variants which targeted specific system aspects of a multi-core machine. For example, Buckwild! \cite{Sa2015TamingTW} proposed using restricted precision training on a CPU. 

As mentioned \hwp \citep{zhang2016hogwild}, has a decentralized design  targeted to NUMA machines, thus, it resembles \alsgd. However, there are significant differences, which we discuss in Appendix \ref{append:morerelated}. 
In the local setting of a GPU, the partial backpropagation of \palsgd derives from \pasmp  \citep{kungurtsev2019asynchronous}, however, unlike them our gradient updates are lock-free. Nonetheless, the similarities end there only -- \palsgd is a decentralized method unlike \pasmp.

Periodic model averaging after several sequential local updates is now a popular approach for distributed SGD: \cite{WoodworthPS20,Stich19,0001DKD19,nadiradze2021decentralized}. This method has been analyzed with various assumptions on the implementation setting and the problem: \cite{0001MR20,WoodworthPSDBMS20,KoloskovaLBJS20}. Yet, unlike us, none of these methods utilize shared-memory asynchrony or partial backpropagation. 
\cite{Wang2020OverlapLA} proposed Overlap-local-SGD, wherein they suggested to keep a model copy at each worker, very similar to \hwp, which is simultaneously averaged when sequential computation for multiple iterations happen locally. They showed by limited experiments that it reduced the communication overhead in a non-iid training case based on \cft, however, not much is known about its performance in general cases.

 \cite{chen2016revisiting} report that asynchrony causes performance degradation based on limited empirical observations on specific hardware. Yet, others \citep{regatti2019distributed} showed that asynchrony favors generalization on a distributed parameter server. Our work empirically, as well as with analytical insights, establishes that asynchronous noise -- from concurrent updates and from distributed averaging -- is a boon for the generalization performance.

\section{Generalization Discussion}

\label{sec:discuss}	

In~\cite{Lin2020DontUL}, the model of SGD algorithms as a form of a Euler-Maruyama discretization of a Stochastic Differential Equation (SDE) presents the perspective that the batch size can correspond to the inverse of the injected noise. 
Whereas distributed SGD combined stochastic gradients as such effectively results in an SGD step with a larger batch size, local SGD, by averaging models rather than gradients, maintains the noise associated with the local small-batch gradients. 
Given the well-established benefits of greater noise in improving generalization accuracy (e.g.~\cite{smith2018bayesian} and others), this presents a heuristic argument as to why local SGD tends to generalize better than distributed SGD. 
In the Appendix we present an additional argument for why local SGD generalizes well. 

However, we see that our particular variation with asynchronous local updates and asynchronous averaging seems to provide additional generalization accuracy above and beyond local SGD. We provide the following explanation as to why this could be the case, again from the perspective of noise as it would appear in a SDE. Let us recall three facts,
\begin{enumerate}[leftmargin=*,nosep,topsep=0pt,nolistsep,noitemsep]
\item The noise appearing as a discretization of the Brownian motion term in the diffusion SDE, and, correspondingly the injected noise studied as a driver for increased generalization in previous works on neural network training is i.i.d.,
\item Clearly, the covariances of minibatch gradients as statistical estimates of the gradients at model estimates $x$ and $x'$ are going to be more similar when $x$ and $x'$ are closer together,
\item Locally asynchronous gradients computed via snapshots taken before a previous all-to-all averaging step, and thus far from the current state of the model in memory, present additional noise, though they provide additional challenge for the convergence theory (see Section~\ref{sec:convergence}).
\end{enumerate}
Thus, effectively, the presence of these ``highly asynchronous" stochastic gradients, while being potentially challenging from the convergence perspective, effectively brings the analogy of greater injected noise for local SGD by inducing greater probabilistic independence, i.e., the injected noise, for these updates, is far close to the i.i.d. noise that appears in a discretized SDE.

\section{Conclusion}

\label{sec:conclude}

Picking from where \cite{Golmant2018OnTC} concluded referring to their findings: \textit{``These results suggest that we should not assume that increasing the batch size for larger datasets will keep training times manageable for all problems. Even though it is a natural form of data parallelism for large-scale optimization, alternative forms of parallelism should be explored to utilize all of our data more efficiently"}, we introduced a fresh approach in the direction to address the challenging trade-off of optimization quality vs. scalability. 

There are a plethora of small scale model and dataset combinations, where the \textit{critical batch size}--after which the returns in terms of convergence per wall-clock time diminish--is small relative to existing system capabilities (\cite{Golmant2018OnTC}). To such cases \palsgd becomes readily useful. 

We observed that for \res-50/\imn training the performance gain of \palsgd over \plsgd is roughly 5\%. However, as other benchmarks showed, predictively, this will reasonably change on higher end GPUs, such as NVIDIA A100 with 40 GB memory and 624 teraflops peak performance. As an outlook, we considerately did not include the distributed scalability techniques such as quantized \citep{Alistarh2017QSGDCS} or sparsified- \citep{SeideFDLY14,AlistarhH0KKR18} communication, or for that matter, decentralized scalability methods such as SGP \citep{AssranLBR19}. These approaches can always supplement a locally scalable scheme, which is the focus of this work. Our future goal includes combining \palsgd and these techniques.

In the experiments, we observed that the system-generated noise in some cases effectively improved the generalization accuracy. The empirical findings suggest that the proposed variant of distributed \sgd has a perfectly appropriate place to fit in the horizon of efficient optimization methods for training deep neural networks. As a general guideline for the applicability of our approach, we would suggest the following: monitor the resource consumption of a GPU that trains a CNN, if there is any sign that the utilization was less than 100\%, try out \palsgd instead of arduously, and at times unsuccessfully, tuning the hyperparameters in order to harness the data-parallelism.

\bibliography{texts/refs}

\newpage
\appendix
\section{Convergence Theory}\label{append:proof}

The convergence analysis of our algorithm requires precise modelling and ordering of concurrent processes in time -- the gradient updates by processes $u^q$ at line \ref{line:modelupdate} and the vector difference updates by process $a^q$ at line \ref{line:averageupdate} for lock-free averaging --  in Algorithm \ref{alg:algo} on the model $x^q$ in the shared-memory setting of a worker $q\in\mathcal{Q}$. To this end, we define the following orders.

\paragraph{Snapshot Order.} The snapshot collections over the shared memory of a worker $q\in\mathcal{Q}$ are ordered by the shared-counter $\mathcal{C}^q$ as read at line \ref{line:updatesnap} in Algorithm \ref{alg:algo}. Correspondingly, the partial gradients, as computed using these snapshots, get ordered too. We define the \textit{snapshot order} of a gradient update as the order by value of $\mathcal{C}^q$ of the snapshot collections. 

\paragraph{Update Order.} The updates at each of the components of the model vector $x^q$ are atomic by means of \fadd. Furthermore, each process $u^q$ updates the components of $x^q$  in the same order (from left to right). This would enable us to order the partial gradient updates by when they write over a common model component. However, given that there are independent lock-free updates over the blocks of the model, we can not use any particular component, for example, the first component of $x^q$, to order the gradient updates. To overcome this issue, we assume that each time before updating the model at line \ref{line:modelupdate} in Algorithm \ref{alg:algo}, an updater process $u^q$ atomically accesses the reference (pointer) to $x^q$  over the shared memory. Thus, we order the updates by their atomic access to this reference. We call this order \textit{the update order}. 

\paragraph{Iteration Order.} Because a single designated process $a^q$ on each worker $q\in\mathcal{Q}$ performs averaging, the synchronization rounds (corresponding the vector difference updates at line \ref{line:averageupdate} in Algorithm \ref{alg:algo}) have a natural order, which is also the order of the major iteration views as defined in Section \ref{sec:experiments}. Now, similar to the update order as defined above, we assume that each time before $a^q$ starts writing over the model $x^q$, it atomically accesses the memory-reference to it. Thereby, we have a well-defined arrangement of the gradient updates before or after a particular averaging update over the shared model. Thus we define the \textit{iteration order} of a gradient update as the dictionary order $(j,t)$ to denote that it is the $t^{th}$ gradient update after the $j^{th}$ averaging round. 

With the above description, each of the minor iteration views as defined in Section \ref{sec:convergence} has well-defined snapshot, update, and iteration orders. With that, we define the following delay model.
\ignore{
\subsection{The Delay Model}

\paragraph{Stochastic Delay Model.} The locally asynchronous gradient computation using inconsistent snapshots of the shared-memory model on each worker introduces delays in terms of number of iterations before the updates.
We describe the formal generic details. Because the read and updates at the level of components are atomic, we have the following:
\begin{equation}\label{eq:minorview}
\begin{array}{l}
v^{q}_{j,t}[e]\in \cup_{k\in\{0,1,\ldots,u\}} \{w^{q}_{k}[e]\},
\end{array}
\end{equation}
where $e\in i(q,j,t)$ and  $i(q,j,t)\subseteq\{1,2,\ldots,d\}$ and $u$ is the update order of minor iteration $(j,t)$.

Here $v^{q}_{j,t}[e]$ and $w^{q}_{k}[e]$ denote the $e^{th}$ components of the inconsistent snapshot $v^{q}_{j,t}$ and the minor iteration view $w^{q}_{k}$, respectively. Essentially, relation (\ref{eq:minorview}) captures the fact that a component of an inconsistent snapshot would belong to a \textit{historical minor iteration view}. Notice that, as the lock-free read and updates of the shared model work at the level of components, we have an advantage of dropping the dependence of the expressions over the blocks $i\in\mathcal{I}$ in the convergence analysis. 
We can now define the event describing relation (\ref{eq:minorview}) as the following. 

\paragraph{Delay Event and Probability.} Define $\mathbb{I}^{q,k,e}_{j,t}$ as the event that $v^{q}_{j,t}[e]=w^{q}_{k}[e]$, and $p^{q,k,e}_{j,t}$ as its probability, for some $j\ge 0$, $0\le t \le K_j^q$, and  $0\le k \le u$, where $u$ is the update order of minor iteration view $(j,t)$. We call $\mathbb{I}^{q,k,e}_{j,t}$ a \textit{delay event}.
Notice that, defining this event captures the possibility that $v^{q}_{j,t}$ could have been collected  at a minor iteration before the synchronization round $j$. 

\paragraph{Assumptions Governing the Asynchronous Computation.} In the non-convex setting,  in general, we can simply hope that the objective function decrease achieved by local asynchronous SGD update with a gradient computed after an averaging step outweighs the potential worst-case increase that might take place due to an update by a gradient computed before the averaging step. In order to derive an expected decrease in the objective, we need to bound the probability of an increase, which means bounding the probability that a component of the snapshot was collected at a minor iteration older than the previous averaging.

The significant quantity that separates these expected favorable and unfavorable iterations is,
\[
\mathbb{P}\left[\cap_e \cup_{l\ge k}^u \mathbb{I}^{q,l,e}_{j,t}\right],
\]
where $u$ is the update order of minor iteration $(j,t)$ and $k$ is the update order of the minor iteration just after the $j^{th}$ synchronization round. This quantity is \emph{the probability that every component of the inconsistent snapshots were due to the minor iterations after  $a^q$ performed synchronization (averaging)}.

We can consider several possibilities for this condition,
\begin{assumption}\label{as:delays}
For all $(j,t,q)$ it holds that
\[
\bar{\mathbf{p}}^{(j,t,q)}:=\mathbb{P}\left[\cap_e \cup_{l\ge k}^u \mathbb{I}^{q,l,e}_{j,t}\right]
\]
satisfies that there exists $p_0>0$ such that $\bar{\mathbf{p}}^{(j,t,q)} > p_0$ for all $(j,t,q)$.
\end{assumption}



\paragraph{Elastic Consistency.} With these, we can make a statement to bound the error in the sense of \textit{Elastic Consistency} of \cite{alistarh2020elastic}. In particular, if $s$ is the snapshot order of minor iteration $(j,t)$, 
it holds that,
\begin{equation}\label{eq:consistency1}
\mathbb{E}\left[\left\|w^{q}_{j,t}-v^{q,s}\right\| | \cap_e \cup_{l\ge k}^u \mathbb{I}^{q,l,e}_{j,t}\right]  \le \alpha^2 B^2.
\end{equation}
Equation (\ref{eq:consistency1}) is essentially a restatement of Equation (\ref{eq:consistency}) wherein the event $\cap_e \cup_{l\ge k}^u \mathbb{I}^{q,l,e}_{j,t}$ used for conditioning is precisely the cases where the snapshots for computing partial stochastic gradients were taken after the last averaging round, i.e., if every component satisfies this condition then the elastic consistency bound holds.

Note that under this condition, the maximum delay in terms of number of local asynchronous SGD updates is $\tau_{max}=\bar{K}:=\max_{j,q} K^q_j$. Thus, following  Appendix B.6 of~\cite{alistarh2020elastic}, the following Lemma defines the Elastic consistency constant:
\begin{lemma}
$B=\sqrt{d}KM$, where $d$ is the dimension of the model, $\bar{K}$ is the maximum minor iteration count defined above, and $M$ is the second moment bound as defined in Equation \ref{eqn:secondmoment}.
\end{lemma}
%


}
\subsection{The Proof Overview} 
The proof follows the structure of the ergodic convergence proof of K-step local SGD given in~\cite{zhou2017convergence}, wherein in each round of averaging there are $\sum K^q_j$ total updates to the model vector associated with the $q\in\mathcal{Q}$ workers.

Insofar as the local asynchronous SGD updates are computed over the snapshots close to the globally synchronized model vector at the last averaging step, there is an expected amount of descent, relative to the value at the last averaging step, achieved from these steps. The task of the convergence theory is to quantify how much descent can be guaranteed, and then this is to be balanced with the amount of possible ascent that could occur from computing SGD updates at snapshots collected far from the one at the last averaging step. 

\ignore{In particular, there is no hope of guaranteeing correspondence of the ``view" and actual parameter vectors, in the nonconvex setting, in cases wherein $v^{q,j}_{t,i} = x^{q,j}_{s,i}$ for $s<0$ (i.e., the stochastic gradients are taken, due to local asynchrony, at model vectors with components which existed in memory before the last averaging step). Here we simply bound the worst-case increase in the objective and scale it based on the probability that such an old parameter is used for the computation. 
}
In particular, we need to bound the worst-case increase in the objective due to SGD updates computed over the snapshots collected before the updates by the last averaging and scale it based on the probability that such an old snapshot is used for the computation.

\paragraph{Good and Bad Iterates.} To balance these two cases, the analysis takes an approach, inspired partially by the analysis given in~\cite{cartis2018global} of separating these as ``good" and "bad" iterates: the "good" iterates correspond to snapshots (all components thereof) for SGD updates at line \ref{line:updatesnap} read after the last snapshot (all components thereof) for averaging at line \ref{line:averagesnap} was taken in Algorithm \ref{alg:algo}, with some associated guaranteed descent in expectation, whereas the ``bad" iterates are those read before that. Thus we split the expected descent as expected descent with a ``good" iterate multiplied by its probability, added with the worst case ascent of a ``bad" iterate multiplied by its probability.

By considering the stochastic process governing the amount of asynchrony as being governed by probabilistic laws, we can characterize the probability of a ``good" and ``bad" iterate and ultimately seek to balance the total expected descent from one, and worst possible ascent in the other, as a function of these probabilities.

\subsection{Proof of Main Convergence Theorem}

Recall the major iteration views $\bar{x}_j$ are defined as the average of the last minor iteration views. Restating its update rule here for reference in this section,
\begin{equation}\label{eq:update1}
\bar x_{j+1}=\frac{1}{Q}\sum_{q=1}^Q\left[\bar x_j-|_i
\alpha \sum_{t=0}^{K_{j}^{q}} \tn_{i} f(v^{q}_{j,t})\right],
\end{equation}
where $|_i$ is the block-restricted subtraction as we defined in Section \ref{sec:convergence}.


We are now ready to prove the convergence Theorem. 
The structure of the proof will follow~\cite{zhou2017convergence}, who
derives the standard sublinear convergence rate for local SGD in a synchronous environment for nonconvex objectives.

\begin{theorem}[Restatement of Theorem (\ref{th:convergence})]\label{th:convergence1}
	Under Assumptions ~\ref{eqn:unbiased},~\ref{eqn:Lower bound},~\ref{eqn:smoothness},~\ref{eqn:secondmoment}, and~\ref{eqn:variance}
	together with Assumptions \ref{as:delays} on delay probabilities, it holds that,
	\begin{equation}\label{eq:conv1}
		\begin{array}{l}
			\frac{1}{Q} \sum_{j=1}^J \sum_{q=1}^Q \sum_{t=0}^{K_j^q}\left\{
			[\alpha C_1-\alpha^2 C_2]\mathbb{E}\left\|\nabla_{i} f(v^{q}_{j,t})\right\|^2+[\alpha C_3-\alpha^2C_4] \mathbb{E}\left\|\nabla_{i} f(\bar{x}_j)\right\|^2 \right\}
			\le f(\bar{x}_0)-f_m
		\end{array}
	\end{equation}
	where $C_1$, $C_2$, $C_3$, and $C_4$ depend on $L$, $B$, $M$, $\sigma$ and probabilistic quantities defining the asynchronous computation. Thus there exists a set of such constants s. t. if $\alpha=\Theta\left(\frac{1}{\sqrt{J}}\right)$
	then Algorithm~\ref{alg:algo} ergodically converges with the standard  $O(1/\sqrt{J})$ rate for nonconvex objectives.
\end{theorem}
\begin{proof}
	
We begin with the standard application of the Descent Lemma,

\begin{equation}\label{eq:descent1b}
\begin{array}{l}
f(\bar x_{j+1})-f(\bar x_j) \le
-\langle \nabla f(\bar x_j),
\frac{1}{Q}\sum_{q=1}^Q \sum_{t=0}^{K_{j}^{q}}\alpha \tn_{i} f(v_{j,t}^{q})\rangle
 +\frac{L}{2}\left\|\frac{1}{Q}\sum_{q=1}^Q \sum_{t=0}^{K_{j}^{q}} \alpha \tn_{i}
f(v_{j,t}^{q})\right\|^2
\end{array}
\end{equation}

Now, since $\mathbb{E}\left[\tn_{i} f(v_{j,t}^{q})\right]=\nabla_{i} f(v_{j,t}^{q})$ using the unbiasedness Assumption \ref{eqn:unbiased},
\begin{equation}\label{eq:nonconvexexpffv}
\begin{array}{l}
-\mathbb{E}\langle \nabla f(\bar x_j),\tn_{i} f(v_{j,t}^{q})\rangle =
-\frac{1}{2}\left(\|\nabla_{i} f(\bar x_j)\|^2+\mathbb{E}\|\nabla_{i} f(v_{j,t}^{q})\|^2  -\mathbb{E}\|\nabla_{i} f(v_{j,t}^{q})-\nabla_{i} f(\bar x_j)\|^2\right)
\end{array}
\end{equation}
We now split the last term by the two cases and use~\Eqref{eq:consistency1},
\[
\begin{array}{l}
\mathbb{E}\|\nabla_{i} f(v_{j,t}^{q})-\nabla_{i} f(\bar x_j)\|^2 \\  =
\mathbb{E}\left[\|\nabla_{i} f(v_{j,t}^{q})-\nabla_{i} f(\bar x_j)\|^2|
\cap_e \cup_{l\ge k}^u \mathbb{I}^{q,l,e}_{j,t}\right]\mathbb{P}\left[\cap_e \cup_{l\ge k}^u \mathbb{I}^{q,l,e}_{j,t}\right]
\\ \quad + 
\mathbb{E}\left[\|\nabla_{i} f(v_{j,t}^{q})-\nabla_{i} f(\bar x_j)\|^2|
\left(\cap_e \cup_{l\ge k}^u \mathbb{I}^{q,l,e}_{j,t}\right)^c\right]\times\mathbb{P}\left[\left(\cap_e \cup_{l\ge k}^u \mathbb{I}^{q,l,e}_{j,t}\right)^c\right]
\\ 
\le \alpha^2 B^2 \mathbb{P}\left[\cap_e \cup_{l\ge k}^u \mathbb{I}^{q,l,e}_{j,t}\right] +
2\left(\|\nabla_{i} f(\bar x_j)\|^2+\mathbb{E}\|\nabla_{i} f(v_{j,t}^{q})\|^2\right) \mathbb{P}\left[\left(\cap_e \cup_{l\ge k}^u \mathbb{I}^{q,l,e}_{j,t}\right)^c\right]
\end{array}
\]
and thus combining with~\eqref{eq:nonconvexexpffv} we get the overall bound,
\begin{equation}\label{eq:nonconvexexpffvprobs}
\begin{array}{l}
-\mathbb{E}\langle \nabla f(\bar x_j),\tn_{i} f(v_{j,t}^{q})\rangle \\ 
\qquad \le -\frac 12 \left(\|\nabla_{i} f(\bar x_j)\|^2+\mathbb{E}\|\nabla_{i} f(v_{j,t}^{q})\|^2 -\alpha^2 B^2\right)\times
\mathbb{P}\left[\cap_e \cup_{l\ge k}^u \mathbb{I}^{q,l,e}_{j,t}\right]  \\ 
\qquad \qquad + \frac 12
\left(\|\nabla_{i} f(\bar x_j)\|^2+\mathbb{E}\|\nabla_{i} f(v_{j,t}^{q})\|^2 \right) \mathbb{P}\left[\left(\cap_e \cup_{l\ge k}^u \mathbb{I}^{q,l,e}_{j,t}\right)^c\right]
\\ \qquad \le -  \left(\|\nabla_{i} f(\bar x_j)\|^2+\mathbb{E}\|\nabla_{i} f(v_{j,t}^{q})\|^2\right) 
\mathbb{P}\left[\cap_e \cup_{l\ge k}^u \mathbb{I}^{q,l,e}_{j,t}\right] \\ \qquad\qquad
+\frac 12
\left(\|\nabla_{i} f(\bar x_j)\|^2+\mathbb{E}\|\nabla_{i} f(v_{j,t}^{q})\|^2 \right)+\frac{\alpha^2 B^2}{2}
\end{array}
\end{equation}

Combining these, we get,
\begin{equation}\label{eq:probboundfullj}
\begin{array}{l}
\frac{1}{Q}\sum_{q=1}^Q \sum_{t=0}^{K_{j}^{q}} \alpha\left[ \mathbb{P}\left[\cap_e \cup_{l\ge k}^u \mathbb{I}^{q,l,e}_{j,t}\right]
-\alpha \left(Q K_{j}^{q}+\frac{B}{2}\right) \right]\times \mathbb{E}\|\nabla_{i} f(v_{j,t}^{q})\|^2 \\ 
 +
\frac{1}{Q}\sum_{q=1}^Q \sum_{t=0}^{K_{j}^{q}}\alpha \left[ \mathbb{P}\left[\cap_e \cup_{l\ge k}^u \mathbb{I}^{q,l,e}_{j,t}\right]
-\frac{\alpha B^2}{2} \right] \mathbb{E}
\|\nabla_{i} f(\bar x_j)\|^2
\\ \le f(\bar{x}_j)-f(\bar{x}_{j+1})
\end{array}
\end{equation}

Summing over $j$ obtains the expression in Theorem \ref{th:convergence1} with $C_1=\mathbb{P}\left[\cap_e \cup_{l\ge k}^u \mathbb{I}^{q,l,e}_{j,t}\right]$, $C_2=Q K_{j}^{q}+\frac{B}{2}$, $C_3=\mathbb{P}\left[\cap_e \cup_{l\ge k}^u \mathbb{I}^{q,l,e}_{j,t}\right]$, and $C_4=\frac{B^2}{2}$, whereof, Assumption~\ref{as:delays} ensures the existence of $C_1$ and $C_3$.
\end{proof}

\textbf{Remark} In order to achieve \emph{asymptotic} convergence, Assumption~\ref{as:delays} can be relaxed further, in particular, it just cannot be summable, i.e., the condition,
\[
\sum_{j=1}^\infty \sum_{q=1}^Q \sum_{t=0}^{K^q_j} 
\mathbb{P}\left[\cap_e \cup_{l\ge k}^u \mathbb{I}^{q,l,e}_{j,t}\right] =\infty
\]
is sufficient for guaranteeing asymptotic convergence.


\paragraph{Discussion.} 
\ignore{Let us now consider the quantity $ \mathbb{P}\left[\cap_e \cup_{l\ge k}^u \mathbb{I}^{q,l,e}_{j,t}\right]$ in more detail and study
how the nature of the concurrency affects the possibility and rate of convergence.

In particular notice that, 
\[
\mathbb{P}\left[\cap_e  \mathbb{I}^{q,j,t,e}_{k,l}\right]
\le 1-p_0
\]
for $k\ge j-\nu$ and  $l\ge t-\tau$. In general of course we expect this quantity to increase as $k$ and $l$ get closer to $j$ and $t$, respectively. 
}
Now, let us consider two extreme cases: if there is always only one local stochastic gradient update for all workers for all major iterations $j$, i.e., $K_{j}^{q}\equiv 1$, \emph{any} delay means
reading a vector in memory before the last major iteration, and thus the probability of delay greater than zero must be very small in order to offset the worst possible ascent. 

On the other hand, if in general $K_{j}^{q}>> \tau$, then while the first $\tau$ minor iteration views could be as problematic at a level depending on the probability of the small delay times, for $t>\tau$ clearly the vector $v^{q}_{j,t}$ satisfies~\Eqref{eq:consistency1}.

Thus we can sum up our conclusions in the following statements:
\begin{enumerate}
    \item Overall, considering distribution for the delays, the higher the mean, variance, and thickness of the tails of this distribution, the more problematic convergence would be,
    \item The larger the quantity of local updates each worker performs in between averaging, the more likely a favorable convergence would occur.
\end{enumerate}
The first is of course standard and obvious. The second presents the interesting finding that if you are running asynchronous stochastic gradient updates on local shared memory, performing local SGD with a larger gap in time between averaging results in more robust performance. 

This suggests a certain fundamental \emph{harmony} between asynchronous concurrency and local SGD, more ``aggressive" locality, in the sense of doing more local updates between averaging, coincides with expected performance gains and robustness of more ``aggressive" asynchrony and concurrency, in the sense of delays in the computations associated with local processes. 

In addition, to contrast the convergence in regards to the block size, which in turn depends on the number of updater processes, clearly the larger the block, equivalently, smaller the number of updater processes the faster the overall convergence, since the norms of the gradient vectors appear. An interesting possibility to consider is if a process can roughly estimate or predict when averaging could be triggered, robustness could be gained by attempting to do block updates right after an expected averaging step, and full parameter vector updates later on in the major iteration.

\section{An Argument for Increased Generalization Accuracy}\label{append:general}
\subsection{Wide and Narrow Wells}
In general it has been observed that whether a local minimizer is wide or narrow, or how ``flat'' it is,
seems to affect its generalization properties~\citep{keskar2019large}. Motivated by investigating the impact
of batch size on generalization,~\cite{dai2018towards} analyzed the generalization properties of SGD by considering
the escape time from a ``well'', i.e., a local minimizer in the objective landscape, for a constant stepsize
variant of SGD by modeling it as an overdamped Langevin-type diffusion process,
\[
dX_t = -\nabla f(X_t)dt+\sqrt{2\epsilon}dW_t
\]
In general ``flatter'' minima have longer escape times than narrow ones, where the escape time
is the expectation in the number of iterations
(defined as a continuous parameter in this sense) until the iterates leave the well to explore the rest
of the objective landscape. Any procedure that increases the escape time for flatter minima as compared to
narrower ones should, in theory, result in better generalization properties, as it is more likely then that
the procedure will return an iterate that is in a narrow minimizer upon termination.

Denote with indexes $w$ for a ``wide'' valley local minimizer and $n$ for a ``narrow'' value, which also
corresponds to smaller and larger minimal Hessian eigenvalues, respectively. 

The work~\cite{berglund2011kramers} discusses the ultimately classical result 
that as $\epsilon\to 0$, the escape time from a local minimizer valley satisfies,
\[
\mathbb{E}[\tau_e] = H e^{C/\epsilon}
\]
and letting the constant $H$ depend on the type of minimizer, 
it holds that that $H_w> H_n$, i.e., this time is longer for a wider valley. 

We also have from the same reference,
\[
\mathbb{P}[\tau_e>s\mathbb{E}[\tau_e]]=e^{-s}
\]

\subsection{Averaging}
We now contrast two procedures and compare the difference in their escape times for narrow and wider local minimizers.
One is the standard SGD, and in one we perform averaging every $\tau_a$ time. In each case there are $Q$ workers, in the first case
running independent instances of SGD, and in the other averaging their iterates. We average the models probabilistically as it 
results in a randomized
initialization within the well, and thus the escape time is a sequence of independent trials of length $\tau_a$ with an initial point in the well,
i.e., escaping at time $\tau_e$ means that there are $Q\left\lceil{\frac{\tau_e}{\tau_a}}\right\rceil$ trials wherein none of the $Q$
sequences escaped within $\tau_a$, and then one of them escaped in the next set of $Q$ trials.

For ease of calculation, let us assume that $\tau_a = \frac{1}{2}\mathbb{E}[\tau^w_e]=2\mathbb{E}[\tau^n_e]$,
where $\tau^w_e$ and $\tau^n_e$ are the calculated single process escape time from a wide and narrow well, respectively.

If any one of the local runs escapes, then there is nothing that can be said
about the averaged point, so a lack of escape is indicated by the case for which all trajectories, while periodically averaged, stay
within the local minimizer value. 

Now consider if no averaging takes place, we sum up the probabilities for the wide valley
that they all escape after time $(i-1)\tau$ time and, given that they do so, not all of them escape after $i\tau_a$.
\[
\begin{array}{l}
\mathbb{E}[\tau^{w}_e] \le \sum_{i=1}^\infty P(\tau^w_e >(i-1)\tau_a)^Q   (1-P(\tau^w_e>\tau_a i|\tau^w_e >(i-1)\tau_a)^Q) i\tau_a
\\ \qquad \le \sum_{i=1}^\infty\left(e^{-\frac{Q(i-1)}{2}}\left(1-e^{-\frac{Q}{2}}\right) \tau_a i\right)
\end{array}
\]
For the narrow well this is,
\[
\begin{array}{l}
\mathbb{E}[\tau^{n}_e] \ge \sum_{i=1}^\infty P(\tau^n_e >(i-1)\tau_a)^Q (1-P(\tau^n_e>\tau_a i|\tau^n_e >(i-1)\tau_a)^Q) (i-1)\tau_a
\\ \qquad \qquad\ge \sum_{i=1}^\infty\left(e^{-2Q(i-1)}\left(1-e^{-2Q}\right) \tau_a (i-1)\right)
\end{array}
\]
The difference in the expected escape times satisfies,
\[
\begin{array}{l}
\mathbb{E}[\tau^{w}_e-\tau^{n}_e] \le \sum_{i=1}^\infty\left[\left[\left(e^{-\frac{Q(i-1)}{2}}\left(1-e^{-\frac{Q}{2}}\right)\right)-\left(e^{-2Q(i-1)}\left(1-e^{-2Q}\right)\right)\right](i-1)\right.\\ \qquad\qquad\qquad\left. +\left(e^{-\frac{Q(i-1)}{2}}\left(1-e^{-\frac{Q}{2}}\right) \right)\right]\tau_a
\end{array}
\]

Recall that in the case of averaging, if escaping takes place between $(i-1)\tau_a$ and $i\tau_a$
there were no escapes with less that $\tau_a$ for $M$ processors multiplied by $i-1$ times trials,
and at least one escape between $(i-1)\tau_a$ and $i\tau_a$, i.e., not all did not escape between these two times.

The expected first escape time for any trajectory among $Q$ from a wide valley, thus, is,
\[
\begin{array}{l}
\mathbb{E}[\tau^{a,w}_e]\le \sum_{i=1}^\infty \left(\mathbb{P}[\tau^{w}_e>\tau_a]^{(i-1)Q}(1-\mathbb{P}[\tau^{w}_e>\tau_a]^{Mi}) \tau_a i\right)
\\ \qquad \le \sum_{i=1}^\infty e^{-\frac{(i-1)Q}{2}}(1-e^{-\frac{iQ}{2}})\tau_a i
\end{array}
\]

And now with averaging, the escape time from a narrow valley satisfies
\[
\begin{array}{l}
\mathbb{E}[\tau^{a,n}_e]\ge \sum_{i=1}^\infty \left(\mathbb{P}[\tau^{n}_e>\tau_a]^{(i-1)Q}(1-\mathbb{P}[\tau^{n}_e>\tau_a]^{Qi}) \tau_a (i-1)\right) \\ \qquad
\ge \sum_{i=1}^\infty e^{-2(i-1)Q}(1-e^{-2iQ})\tau_a (i-1).
\end{array}
\]

Thus the difference in the expected escape times satisfies
\[
\begin{array}{l}
\mathbb{E}[\tau^{a,w}_e-\tau^{a,n}_e] \le \sum_{i=1}^\infty\left[\left[e^{-\frac{(i-1)Q}{2}}(1-e^{-\frac{iQ}{2}})-e^{-2(i-1)Q}(1-e^{-2iQ})\right](i-1)+
e^{-\frac{(i-1)Q}{2}}(1-e^{-\frac{iQ}{2}})\right]\tau_a.
\end{array}
\]

It is clear from the above expressions that the upper bound for the difference is larger in the case of averaging. This implies that
averaging results in a greater difference between the escape times between wider and narrow local minimizers, suggesting that, on average
if one were to stop a process of training and use the resulting iterate as the estimate for the parameters, this iterate
would more likely come from a flatter local minimizer if it was generated with a periodic averaging procedure, relative to standard SGD. 
Thus it should be expected, at least by this argument that better generalization is more likely with periodic averaging.

Note that since they are both upper bounds, this isn't a 
formal proof that in all cases the escape times are more favorable
for generalization in the case of averaging, but a guide
as to the mathematical intuition as to how this could be the case.

\section{More on Related Work}\label{append:morerelated}
The main contribution of the present work is the algorithm \palsgd, where the shared models on each worker receive partial asynchronous gradient updates locally in addition to being averaged across the workers in a non-blocking fashion. On the other hand, in \alsgd each local gradient update is on the full model vector. In the existing literature, a variant of celebrated \hw \citep{recht2011hogwild}, called \hwp \citep{zhang2016hogwild} was proposed, which harnesses the non-uniform-memory-access (NUMA) architecture based multi-core computers. In this method, threads \textit{pinned} to individual CPUs on a multi-socket mainboard with access to a common main memory, form clusters. 
In principle, the proposed \alsgd may resemble \hwp. However, there are important differences: 
\begin{enumerate}
	\item At the structural level, the averaging in \hwp is binary on a ring graph of thread-clusters. Furthermore, it is a token based procedure where in each round only two neighbours synchronize, whereas in \alsgd it is all-to-all.
	\item In \hwp each cluster maintains two copies of the model: a locally updating copy and a buffer copy to store the last synchronized view of the model, whereby each cluster essentially passes the ``update" in the local model since the last synchronization to its neighbour. However, this approach has a drawback as identified by the authors: the update that is passed on a ring of workers eventually \enquote{comes back} to itself thereby leading to divergence, to overcome this problem they decay the sent out update. By contrast, \alsgd uses no buffer and does not track updates as such. Averaging the model with each peer, similar to \lsgd, helps each of the peers to adjust their optimization dynamics.
	\item It is not known if the token based model averaging of \hwp is sufficient for training deep networks where generalization is the core point of concern. As against that, we observed that our asynchronous averaging provides an effective protocol of synchronization and often results in improving the generalization.
	\item Comparing the thread-clusters of \hwp to concurrent processes on GPUs in \alsgd, the latter uses a dedicated process that performs averaging without disturbing the local gradient updates thereby maximally reducing the communication overhead.
	\item Finally, we prove convergence of \alsgd for non-convex problems, by contrast, \hwp does not have any convergence guarantee.  
\end{enumerate}

Thus, in its own merit, \alsgd is a novel distributed \sgd method. However, to our knowledge, this is the first work to suggest averaging the models on distributed workers, where locally they are asynchronously updated with partial stochastic gradients, which is the proposed method \palsgd.  

\section{Experiments and Hyperparameters}\label{append:exp}
\paragraph{System Details.} More detail on the specific settings: (a) \textbf{S1:} two Nvidia GeForce RTX 2080 Ti GPUs and an Intel(R) Xeon(R) E5-1650 v4 CPU running @ 3.60 GHz with 12 logical cores, (b) \textbf{S2:} four Nvidia GeForce RTX 2080 Ti GPUs and two Intel(R) Xeon(R) E5-2640 v4 CPUs running @ 2.40 GHz totaling 40 logical cores, (c) \textbf{S3:} two \textbf{S2} settings connected with a 100 GB/s infiniband link, and (d) \textbf{S4:} eight Nvidia V100 GPUs, 2 each on 4 nodes in a supercomputing cluster; each node contains two Intel(R) Xeon(R) CPUs running @ 2.70 GHz totaling 72 logical cores; the nodes on the cluster are managed by slurm and connected with 100 GB/s infiniband links.

\paragraph{Nvidia Multi-process Service.} It is pertinent to mention that the key enabler of the presented fully non-blocking asynchronous implementation methodology is the support for multiple independent client connection between a CPU and a GPU. Starting from early 2018 with release of Volta architecture, Nvidia's technology Multi-process Service (MPS) efficiently support this. For more technical specifications please refer to their doc-pages \cite{mpsnvidia}. Now, we present extra experimental results. 

\paragraph{Hyperparameters' Details.} The hyperparameters of the experiments are provided in Table \ref{tab:hyper}. In each of the experiments, we have a constant momentum of 0.9. In each of them, as mentioned before, we use standard multi-step diminishing learning rate scheduler with dampening factor $\gamma$ at designated epochs for the methods \mbsgd and \plsgd, whereas, we use cosine annealing for the methods \alsgd and \palsgd without restart. Cosine annealing starts immediately after warm-up, which in each case happens for the first 5 epochs. For \palsgd, we fix initially $E_{st}$ number of epochs when we perform full asynchronous SGD updates i.e. \alsgd. 

\textbf{More \cft/\cfh Results.} Performance of the methods for \res-20/\cfh for 300 epochs and \sqn/\cfh for 200 epochs are available in Tables \ref{tab:rntcfh} and \ref{tab:sqnetcfh}, respectively. In \alsgd, \palsgd, and \plsgd we use $K=1$ up to 150 and 100 epochs in the two respective experiments and after that $K=16$.  In summary, as seen before, the relative latency of the methods are as seen in Table \ref{tab:scale}, whereas in each case \alsgd and \palsgd recovers or improves the baseline training results.
\begin{table*}[h]
	\centering\small
	\begin{tabularx}{\textwidth}{|X|X|X|X|X|X|X|X|}	
	\toprule
	\textbf{Method} & $B_{loc}$ & $U$ & \textbf{Tr.L.} & \textbf{Te.L.} & \textbf{Tr.A.} & \textbf{Te.A.} & \textbf{$T$} \\
	\midrule
	\alsgd &      128 &           6 &     0.295 &    1.191 &      {92.23} &     {69.43} &     {1285} \\     
	\alsgd &      128 &           4 &     0.299 &    1.198 &      {92.15} &     {69.78} &     {1325} \\     
	\palsgd &      128 &         6 &     0.404 &    1.173 &      88.69 &     \textbf{69.47} &     \textbf{1085} \\    
	\palsgd &      128 &           4 &     0.378 &    1.161 &      89.50 &     \textbf{69.75} &     \textbf{1154} \\     
	\mbsgd &     1024 &              - &     0.464 &    1.212 &      86.46 &     67.04 &     1244 \\
	\mbsgd &      128 &             - &     0.360 &    1.111 &      89.93 &     69.73 &     1754 \\
	\plsgd &     1024 &                 - &     0.373 &    1.152 &      89.48 &     67.97 &     1198 \\
	\plsgd &      128 &           - &     0.379 &    1.099 &      89.29 &     69.67 &     1613 \\
	\bottomrule
\end{tabularx} \caption{\small \rnt on \cfh over the setting \textbf{S1}.}

	\label{tab:rntcfh}\end{table*}

\begin{table*}[h]
	\centering\small
	\begin{tabularx}{\textwidth}{|X|X|X|X|X|X|X|X|}		
	\toprule
	\textbf{Method} & $B_{loc}$ & $U$ & \textbf{Tr.L.} & \textbf{Te.L.} & \textbf{Tr.A.} & \textbf{Te.A.} & \textbf{$T$}\\
	\midrule
 \alsgd &      128 &                          4 &     0.037 &    1.549 &      {99.56} &     {70.17} &     {1263} \\
\palsgd &      128 &                  4 &     0.070 &    1.504 &      98.79 &     \textbf{69.97} &     \textbf{1071} \\ 
 \mbsgd &     1024 &                - &     0.043 &    1.294 &      99.60 &     68.30 &     1584 \\
 \mbsgd &      128 &                   - &     0.049 &    1.178 &      99.29 &     70.06 &     2134\\
 \plsgd &     1024 &                    - &     0.061 &    1.259 &      99.25 &     68.46 &     1209 \\
  \plsgd &      128 &                  - &     0.062 &    1.207 &      99.00 &     69.38 &     1857\\  
\bottomrule
\end{tabularx} \caption{\sqn on \cfh over the setting \textbf{S1}.}

	\label{tab:sqnetcfh}\end{table*}
\vspace{0.3cm}
\begin{figure*}[h]
	\setlength\tabcolsep{3pt} 
	\centering
	\small
	\begin{tabularx}{\textwidth}{|p{1.2cm}|p{2.8cm}|p{1.5cm}|p{0.7cm}|p{0.7cm}|p{0.8cm}|X|X|X|p{0.5cm}|p{2cm}|X|p{0.5cm}|}
		\hline
		Exp. & Model & Dataset & Epochs	& $\underline{B_{loc}}$	& $\underline{\alpha}$& $\alpha^w_{DP}$& $\alpha^w_{LAP}$& $\alpha^w_{LPP}$ & $\gamma$  & Dampen -ing at & Wt. decay &  $T_{st}$ \\
		\hline
		Table \ref{tab:sgd}& \res-20 & \cft	& 300 & 128 &  0.1  & 0.2  & 0.2  & 0.25  & 0.1 & 150, 225 & 0.0005 & 50 \\
		\hline
		Table \ref{tab:scale}& \wres-16x8 & \cft	& 200 & 128  &  0.1  & --  & --  & --   & 0.2 & 60, 120, 160 & 0.0001 & 20 \\
		\hline
		Table \ref{tab:Arch_wnet168_Data_cifar100}& \wres-16x8 & \cfh & 200 & 128  &  0.1  & 0.8  & 0.8  & 1.0   & 0.2 & 60, 120, 160 & 0.0001 & 20 \\
		\hline
		Table \ref{tab:Arch_squuezenet_Data_cifar10}& \sqn & \cft & 200 & 128  &  0.1  & 0.2  & 0.2  & 0.25   & 0.2 & 60, 120, 160 & 0.0001 & 20 \\
		\hline
		Table \ref{tab:rn50}& \res-50 & \imn	& 90  & 32 &  0.0125  & 0.8  & 0.1  & 0.125  & 0.1 & 30, 60, 80 & 0.00005 & 9 \\
		\hline
		Table \ref{tab:rntcfh}& \res-20 & \cfh	& 300  & 128 &  0.1  & 0.2  & 0.2  & 0.25   & 0.1 & 150, 225 & 0.0005 & 50 \\
		\hline
		Table \ref{tab:sqnetcfh}& \sqn & \cfh & 200  & 128 &  0.1  & 0.2  & 0.2  & 0.25   & 0.2 & 60, 120, 160 & 0.0001 & 20 \\
		\hline		
	\end{tabularx}\captionof{table}{Hyperparameters used in the experiments.  $\underline{\alpha}$: Initial Learning Rate (LR), $\alpha^w_{DP}$: Warmed-up LR for Data-parallel methods \mbsgd and \plsgd \textit{with standard batch size $\underline{B_{loc}}$ on each GPU} (for large local batch sizes we scale the LR in the ratio of used batch size and $\underline{B_{loc}}$), $\alpha^w_{LAP}$: Warmed-up LR for \alsgd, $\alpha^w_{LPP}$: Warmed-up LR for \palsgd. In general, we follow the standard LR warm-up strategy as discussed before: warm up for the first 5 epochs to scale from $\underline{\alpha}$ to $\underline{\alpha}\times\frac{B_{loc}{\times}Q}{\underline{B_{loc}}}$, where $Q$ is the number of workers. ``Dampening at'' are the epoch counts where we dampen $\alpha^w_{DP}$ by the multiple $\gamma$ following the standard multi-step diminishing LR tuning scheme.}
\label{tab:hyper}\end{figure*}

	
\end{document}